\documentclass{article}

\usepackage{graphicx}
\usepackage{subfigure}

\usepackage[final]{neurips_2020}

\usepackage[utf8]{inputenc} % allow utf-8 input
\usepackage[T1]{fontenc}    % use 8-bit T1 fonts
\usepackage{url}            % simple URL typesetting
\usepackage{booktabs}       % professional-quality tables
\usepackage{amsfonts}       % blackboard math symbols
\usepackage{nicefrac}       % compact symbols for 1/2, etc.
\usepackage{microtype}      % microtypography

\usepackage{amsmath,amssymb,amsthm}

 \usepackage[ruled,vlined,linesnumbered,noresetcount ]{algorithm2e}
\usepackage{mathtools}  
\usepackage{tabulary}
\usepackage{booktabs}

\usepackage{graphics}
\usepackage{ textcomp }

\usepackage{pdfpages}
 \usepackage{colortbl}
 \usepackage[english]{babel}
 \usepackage[utf8]{inputenc}
 \usepackage{lipsum}
\usepackage{ nicefrac}
\usepackage{bbm} 
\usepackage{comment}

\usepackage{color}

\usepackage[hidelinks]{hyperref}
\definecolor{darkred}{RGB}{150,0,0}
\definecolor{darkgreen}{RGB}{0,150,0}
\definecolor{darkblue}{RGB}{0,0,150}
\hypersetup{colorlinks=true, linkcolor=red, citecolor=blue, urlcolor=darkblue}

\newtheorem{assumption}{Assumption}
\newtheorem{theorem}{Theorem}[section]

\newtheorem{lemma}[theorem]{Lemma}
\newtheorem{Proposition}[theorem]{Proposition}

\newtheorem{remark}{Remark}[section]
\theoremstyle{definition}
\newtheorem{definition}{Definition}[section]

\newcommand{\norm}[1]{\left\lVert#1\right\rVert}

\title{Stage-wise Conservative Linear Bandits}

\author{
  Ahmadreza Moradipari, Christos Thrampoulidis, Mahnoosh Alizadeh   \\
  Department of Electrical and Computer Enginnering\\
  University of California, Santa Barbara\\
  \texttt{ahmadreza$\_$moradipari@ucsb.edu} \\
  }
\begin{document}

\maketitle

\begin{abstract}
We study stage-wise conservative linear stochastic bandits: an instance of bandit optimization, which accounts for (unknown) ``safety constraints" that appear in applications such as online advertising and medical trials. At each stage, the learner must choose actions that not only maximize cumulative reward across the entire time horizon, but further satisfy a linear baseline constraint that takes the form of a \emph{lower bound} on the instantaneous reward. For this problem, we present two novel algorithms, \textit{stage-wise conservative linear Thompson Sampling} (SCLTS) and \textit{stage-wise conservative linear UCB} (SCLUCB), that respect the baseline constraints and enjoy probabilistic regret bounds of order $\mathcal{O}(\sqrt{T} \log^{3/2}T)$ and $\mathcal{O}(\sqrt{T} \log T)$, respectively. Notably, the proposed algorithms can be adjusted with only minor modifications to tackle different problem variations, such as, constraints with bandit-feedback, or an unknown sequence of baseline actions.  We discuss these and other improvements over the state-of-the art. For instance, compared to existing solutions, we show that SCLTS plays the (non-optimal) baseline action at most $\mathcal{O}(\log{T})$ times (compared to $\mathcal{O}(\sqrt{T})$). Finally, we make connections to another studied form of ``safety constraints" that takes the form of an \emph{upper bound} on the instantaneous reward. While this incurs additional complexity to the learning process as the optimal action is not guaranteed to belong to the ``safe set'' at each round, we show that SCLUCB can properly adjust in this setting via a simple modification.

\end{abstract}

\section{Introduction}\label{sec:introduction}

With the growing range of applications of bandit algorithms for safety critical real-world systems, the demand for safe learning is receiving increasing attention \cite{9125942}.  In this paper, we investigate the effect of stage-wise safety  constraints on the linear stochastic bandit problem. Inspired by the earlier work of \cite{vanroy,wu2016conservative}, the type of safety constraint we consider in this paper was first introduced by \cite{khezeli2019safe}. As with the classic linear stochastic bandit problem, the learner wishes to choose a sequence of actions $x_t$ that maximize the expected reward over the horizon. However, here the  learner is also given a baseline policy that suggests an action with a guaranteed level of expected reward at each stage of the algorithm. This could be based on historical data, e.g., historical ad placement or medical treatment policies with known success rates. The safety constraint imposed on the learner requires her to ensure that  the expected reward of her chosen action at every single round be no less than a predetermined fraction of the expected reward of the action suggested by baseline policy. 
{An example that might benefit from the design of stage-wise conservative learning algorithms arises in recommender systems, where the recommender might wish to avoid   recommendations that are extremely disliked by the users at any single round.
Our proposed stage-wise conservative constraints ensures that at no round would the recommendation system cause severe dissatisfaction for the user, and the reward of action employed by the learning algorithm, if not better, should be close to that of baseline policy. Another example is in clinical trials where the effects of different therapies on patients' health are initially unknown. We can consider the baseline policy to be treatments that have been historically employed, with known effectiveness. The proposed stage-wise conservative constraint guarantees that at each stage, the learning algorithm suggests an action (a therapy) that achieves the expected reward  close to that of the baseline treatment, and as such, this experimentation does not cause harm to {\it any single patient's health}.}
To tackle this problem, \cite{khezeli2019safe} proposed a greedy algorithm called SEGE. They use the decomposition of the regret  first proposed in \cite{vanroy}, and show an upper bound of order $\mathcal{O}(\sqrt{T})$ over the number of times that the learning algorithm plays the baseline actions, overall resulting in an expected regret of $\mathcal{O}(\sqrt{T} \log T)$. For this problem, we present two algorithms, SCLTS and SCLUCB, and we provide regret bounds of order $\mathcal{O}(\sqrt{T} \log^{3/2}T)$ and $\mathcal{O}(\sqrt{T} \log T)$, respectively. As it is explained in details in Section \ref{sec:regret_analysis}, we improve the result of \cite{khezeli2019safe}, i.e., we show our proposed algorithms play the (non-optimal) baseline actions at most $\mathcal{O}(\log{T})$ times, while also relaxing a number of  assumptions made in \cite{khezeli2019safe}.  Moreover, we show that our proposed algorithms are adaptable with minor modifications to other safety-constrained variations of this problem. This includes the case where the constraint has a different unknown parameter than the reward function with bandit feedback (Section  \ref{sec:cons_bandit_feedback}), as well as   the setting where the reward of baseline action is unknown to the learner in advance (Section \ref{unkown_baseline_reward}).

\subsection{Conservative Stochastic Linear bandit (LB) Problem with Stage-wise Constraints}\label{setting}
\textbf{Linear Bandit.} The learner is given a convex and compact set of actions $\mathcal{X} \subset \mathbb{R}^d$. At each round $t$, she chooses an action $x_t$ and observes a random reward \begin{align}
    y_t = \langle x_t, \theta_{\star}\rangle   + \xi_t, \label{reward_signal}
\end{align}
where $\theta_{\star} \in \mathbb{R}^d$ is  \textit{unknown} but fixed reward parameter and $\xi_t$ is  zero-mean additive noise. We let $r_t$  be the expected reward of action $x_t$ at round $t$, i.e., $r_t := \mathbb{E}[y_t] = \langle x_t, \theta_{\star}\rangle $.   

\textbf{Baseline actions and stage-wise constraint.} We assume that the learner is given a baseline policy such that  selecting the baseline action $x_{b_t}$ at round $t$, she would receive an expected reward $r_{b_t} := \langle x_{b_t}, \theta_{\star} \rangle$. We assume that the learner knows the expected reward of the actions chosen by the baseline policy.  
We further assume that the learner's action selection rule is subject to a stage-wise conservative constraint of the form\footnote{In Section \ref{sec:cons_bandit_feedback}, we show that our results  also extend to constraints of the form $ \langle x_t, \mu_{\star} \rangle \geq (1-\alpha) q_{b_t},$ where $\mu_{\star}$ is an additional unknown parameter. In this case, we assume the learner receives additional bandit feedback on the constraint after each round.} \begin{align}
r_t =  \langle x_t, \theta_{\star}\rangle  \geq (1-\alpha) r_{b_t}, \label{cons:safety}
\end{align} that needs to be satisfied at each round $t$. In particular, constraint \eqref{cons:safety}
 guarantees that at each round $t$, the expected reward of the action chosen by the learner stays above the predefined fraction  $1-\alpha \in (0,1)$ of the baseline policy. { The parameter $\alpha$, controlling the conservatism level of the learning process, is assumed known to the learner similar to \cite{vanroy,wu2016conservative}}. At each round $t$, an action is called \textit{safe} if its expected reward is above the predetermined fraction of the baseline policy, i.e.,  $(1-\alpha) r_{b_t}$. 
 
 \begin{remark}
 It is reasonable to assume that the leaner has an accurate estimate of the  expected reward of the actions chosen by baseline policy \cite{vanroy}. However, in Section \ref{unkown_baseline_reward}, we relax this assumption, and propose an algorithm to the case where the  expected rewards of the actions chosen by baseline policy are unknown to the learner in advance. 
 \end{remark}

\textbf{Regret.} The \textit{cumulative pseudo-regret} of the learner  up to round $T$ is defined as $R(T) = \sum_{t=1}^T \langle x_{\star}, \theta_{\star} \rangle - \langle x_t, \theta_{\star} \rangle$, 
% \begin{align}
%     R(T) = \sum_{t=1}^T \langle x_{\star}, \theta_{\star} \rangle - \langle x_t, \theta_{\star} \rangle, \label{def:regret}
% \end{align} 
where $x_{\star}$ is the optimal safe action that maximizes the expected reward, \begin{align} 
x_{\star} = \arg\max_{x \in \mathcal{X}} \langle x, \theta_{\star} \rangle. 
\end{align} 
% We further assume that $x_{\star}$ is safe at all rounds. 
The learner's objective is to minimize the pseudo-regret, while respecting the stage-wise conservative constraint in \eqref{cons:safety}. For the rest of the paper, we use regret to refer to the pseudo-regret $R(T)$. 

\subsection{Previous work}

\textbf{Multi-armed Bandits.}
The multi-armed bandit (MAB) framework has been studied in sequential decision making problems under uncertainty. In particular, it captures the exploration-exploitation trade-off, where the learner needs to sequentially choose arms in order to maximize her reward over time while exploring to improve her estimate of the reward of each arm \cite{bubeck2016multi}. 
Two popular heuristics exist for MAB:  Following the \textit{optimism in face of uncertainty} (OFU) principle \cite{Auer, li2017provably, filippi2010parametric}, the so-called Upper Confidence Bound (UCB) based approaches  choose the best feasible action- environment pair according to their current confidence regions on the unknown parameter, and Thompson Sampling (TS) \cite{thompson1933likelihood,kaufmann2012thompson, russo2016information,moradipari2018learning}, which randomly samples the environment and plays the corresponding optimal action.

\textbf{Linear Stochastic Bandits.} 
There exists a rich literature on linear stochastic bandits. 
 Two well-known efficient algorithms for LB are Linear UCB (LUCB) and Linear Thompson Sampling (LTS). For LUCB, \cite{Dani08stochasticlinear,Tsitsiklis,abbasi2011improved} provided a regret guarantee of order $\mathcal{O}(\sqrt{T} \log T)$. For LTS, \cite{agrawal2013thompson, abeille2017linear} provided a regret bound of order $\mathcal{O}(\sqrt{T} \log^{3/2} T)$ in a frequentist setting, i.e., when the unknown parameter $\theta_\star$ is a fixed parameter. We need to note that none of the aforementioned heuristics can be directly adopted in the conservative setting. However, note that the regret guarantee provided by our extensions of LUCB and LTS for the safe setting matches those stated for the original setting. 
%(i.e., in the absence of the constraint)

\textbf{Conservativeness and Safety.} The baseline model adopted in this paper was first proposed in \cite{vanroy,wu2016conservative} in the case of {\it cumulative constraints} on the reward. In \cite{vanroy,wu2016conservative}, an action is considered feasible/safe at round $t$ as long as it keeps the cumulative reward up to round $t$ above a given fraction of a given baseline policy. This differs from our setting, which is focused on stage-wise constraints, where we want the expected reward of the {\it every single action}  to exceed a given fraction of the baseline reward at each time $t$. This is a tighter constraint than that of \cite{vanroy,wu2016conservative}.
The setting considered in this paper was first studied in \cite{khezeli2019safe}, which proposed  an algorithm  called SEGE to guarantee the satisfaction of the  safety constraint at each stage of the algorithm.  While our paper is motivated by \cite{khezeli2019safe}, there are a few key  differences:
1) We prove an upper bound of order $\mathcal{O}(\log T)$ for the number of times that the learning algorithm plays the conservative actions which is an order-wise improvement with respect to that of  \cite{khezeli2019safe}, which shows an upper bound of order $\mathcal{O}(\sqrt{T})$;  2) In our setting, the action set is assumed to be a general convex and compact set in $\mathbb{R}^d$. However, in \cite{khezeli2019safe}, the proof relies on the action set being a specific ellipsoid; 3) In Section \ref{unkown_baseline_reward}, we provide a regret guarantee for the learning algorithm for the case where the baseline reward is unknown. However, the results of \cite{khezeli2019safe} have not been  extended to this case; 4) In Section \ref{sec:cons_bandit_feedback}, we also modify our proposed algorithm and provide a regret guarantee for the case where the constraint has a different unknown parameter than the one in the reward function. However, this is not discussed in \cite{khezeli2019safe}. Another difference between the two works is on the type of performance guarantees. In \cite{khezeli2019safe}, the authors bound the \textit{expected} regret. Towards this goal, they manage to quantify the effect of the risk level $\delta$ on the regret and constraint satisfaction. However, it appears that the analysis in \cite{khezeli2019safe} is limited to ellipsoidal action sets. Instead, in this paper, we present a bound on the regret that holds with high (constant) probability (parameterized by $\delta$) over \textit{all} $T$ rounds of the algorithm. This type of results is very common in the bandit literature, e.g. \cite{abbasi2011improved,Dani08stochasticlinear}, and in the emerging safe-bandit literature \cite{vanroy,amani2019linear,sui2018stagewise}. 
% That being said, we note that our analysis shows that at each stage, the algorithm chooses the safe action with high probability, and we provide a tail bound on regret. However, in \cite{khezeli2019safe} they upper bound the expected regret, and discuss the effect of risk level, i.e., $\delta$, on  regret  as well as constraint satisfaction under special risk level. We need to note that similar discussion would be applicable to our algorithm with minor modifications. 

Another variant of safety w.r.t a baseline policy has also been studied in \cite{mansour2015bayesian, katariya2018conservative} in the multi-armed bandits framework. Moreover, there has been an increasing attention on studying the effect of safety constraints in the Gaussian process (GP) optimization literature.  For example, \cite{Krause,sui2018stagewise}  study the problem of \emph{nonlinear} bandit optimization with nonlinear constraints   using GPs (as non-parametric models). The algorithms in \cite{Krause,sui2018stagewise}  come with convergence guarantees but no regret bound.
Moreover, \cite{ostafew2016robust,7039601} study  safety-constrained optimization using GPs in robotics applications.  
A large body of work has considered safety in the context of model-predictive control, see, e.g., \cite{aswani2013provably, koller2018learning} and references therein. 
Focusing specifically on linear stochastic bandits, extension of  UCB-type algorithms to provide safety guarantees with provable regret bounds was considered recently in \cite{amani2019linear}. This work considers the effect of a linear constraint of the form $x^{\top} B \theta_\star \leq C,$ where $B$ and $C$ are respectively a known matrix and positive constant, and provides a problem dependent regret bound for a safety-constrained version of LUCB that depends on the location of the optimal action in the safe action set. Notice that this setting requires the linear function $x^{\top} B \theta_\star$ to remain below a threshold $C$, as opposed to our setting which considers a lower bound on the reward. We note that the algorithm and proof technique in  \cite{amani2019linear} does not extend  to our setting and would only work for inequalities of the given form; however, we  discuss how our algorithm can be modified to provide   a regret bound of order $\mathcal{O}(\sqrt{T} \log T)$ for the setting of \cite{amani2019linear} in Appendix \ref{SCLUCB2:COMPARISON_SECTION}. A TS variaent of this setting has been studied in \cite{9053865,moradipari2019safe}

\subsection{Model Assumptions}
 \textbf{Notation.}  The weighted $\ell_2$-norm with respect to a positive semi-definite matrix $V$ is denoted by $\|x\|_V = \sqrt{x^{\top} V x}$. The minimum of two numbers $a,b$ is denoted $a \wedge b$.
Let $\mathcal{F}_t = (\mathcal{F}_1, \sigma ( x_1,\xi_1, \dots, x_t,\xi_t))$  be the filtration ($\sigma$-algebra) that represents the information up to round $t$. 

\begin{assumption}\label{ass:sub-gaussian_noise}
For all $t$, $\xi_t$ is conditionally zero-mean R-sub-Gaussian noise variables, i.e., $\mathbb{E}[\xi_t | \mathcal{F}_{t-1}] = 0$, and $\mathbb{E}[e^{\lambda \xi_t} | \mathcal{F}_{t-1}] \leq \exp{(\frac{\lambda^2 R^2}{2})}, \forall \lambda \in \mathbb{R}^d$.
\end{assumption}

\begin{assumption}\label{ass:bounded_parameter}
There exists a positive constant $S$ such that $\|\theta_{\star} \|_2 \leq S$.
\end{assumption}

\begin{assumption}\label{ass:actionset}
The action set $\mathcal{X}$ is a compact and convex subset of $\mathbb{R}^d$ that contains the unit ball. We assume that $\|x\|_2 \leq L, \forall x \in \mathcal{X}$. Also, we assume $\langle x, \theta_{\star} \rangle \leq 1, \forall x \in \mathcal{X}$. 
\end{assumption}

% \begin{assumption}\label{assumption:value_of_optiomal_action_on_constraint}
% The optimal action $x_{\star}$ is safe at all rounds, i.e.,  \begin{align}
%     \langle x_\star , \theta_\star \rangle \geq r_{b_t}, ~ \forall t \geq 1.
% \end{align}
% \end{assumption}

Let $\kappa_{b_t} = \langle x_{\star}, \theta_{\star} \rangle - r_{b_t} $ be the difference  between expected reward of the optimal and baseline actions at round $t$. As in \cite{vanroy}, we assume the following.

\begin{assumption}\label{ass:lowerbound_reward}
There exist $0 \leq \kappa_{l} \leq \kappa_h$ and $0 < r_l  \leq r_h$ such that, at each round $t$ \begin{align}
    \kappa_l \leq \kappa_{b_t} \leq \kappa_h ~ \text{and} ~ r_l \leq r_{b_t}\leq r_h. \label{lowerbound_on_reward_of_baseline_reward}
\end{align}
\end{assumption}

{ We note that since these parameters are associated with the baseline policy, it can be reasonably assumed that they can be estimated accurately from data. This is because we think of the baseline policy as ``past strategy'', implemented before bandit-optimization, thus producing large amount of data.} The lower bound $0 <r_l \leq r_{b_t}$ on the baseline reward  ensures a minimum level of performance at each round. $\kappa_h$ and $r_h$ could be at most 1, due to  Assumption \ref{ass:actionset}. For simplicity, we assume  the lower bound $\kappa_l$ on the sub-optimality gap $\kappa_{b_t}$ is known. If not, we can always choose $\kappa_l = 0$ by optimality of $x_\star$. 

\section{Stage-wise Conservative Linear Thompson Sampling (SCLTS) Algorithm}

In this section we propose a TS variant algorithm in a frequentist setting referred to as \textit{Stage-wise Conservative Linear Thompson Sampling} (SCLTS) for the problem setting in Section \ref{setting}. Our adoption of TS  is due to its well-known computational efficiency over UCB-based algorithms,
since action selection via the latter involves solving optimization problems with bilinear objective functions,
whereas the former would lead to linear objectives. 
However, this choice does not fundamentally affect our approach. In fact, in Appendix \ref{SCLUCB}, we propose a Stage-wise Conservative Linear UCB (SCLUCB) algorithm, and we provide the regret guarantee for it. In particular, we show a regret of order $\mathcal{O}\left( d \sqrt{T} \log(\frac{T L^2}{\lambda \delta}) \right)$ for SCLUCB, which has the same order as the lower bound proposed for LB in \cite{Dani08stochasticlinear, Tsitsiklis}.

At each round $t$, given a regularized least-square (RLS) estimate of $\hat{\theta}_t$, SCLTS samples a perturbed parameter $\tilde{\theta}_t$ with an appropriate distributional property. Then, it searches for the  action that maximizes the expected reward considering the parameter $\tilde{\theta}_t$ as the true parameter while respecting the safety constraint \eqref{cons:safety}. If any such action exists, it is played under certain conditions; else, the algorithm resorts to playing a perturbed version of the baseline action that satisfies the safety constraint.
In order to guarantee constraint satisfaction (a.k.a safety of actions), the algorithm builds a confidence region $\mathcal{E}_t$ that contains the unknown parameter $\theta_\star$ with high probability. Then, it constructs an \textit{estimated safe} set $\mathcal{X}_t^s$ such that all actions $x_t \in \mathcal{X}_t^s$ satisfy the safety constraint for all $v \in \mathcal{E}_t$. The summary of the SCLTS presented in Algorithm \ref{alg:Safe-Linear-TS1}, and a detailed explanation follows.

\begin{algorithm} 

\caption{Stage-wise Conservative Linear Thompson Sampling (SCLTS) }\label{alg:Safe-Linear-TS1}

\textbf{Input:}  $\delta, T, \lambda, \rho_1$

Set $\delta' = \frac{\delta}{4T}$\\
\For{$t=1,\dots,T$ } { 
Sample $\eta_t \sim \mathcal{H}^{\text{TS}}$ \\

Compute RLS-estimate $\hat{\theta}_t$ and $V_t$ according to \eqref{RLS-estimate}

Set $\tilde{\theta}_t = \hat{\theta}_t + \beta_t V_t^{-1/2} \eta_t$\\

 Build the confidence region $\mathcal{E}_t(\delta')$ in \eqref{confidence_region_unkown_reward_parameter} 
% \\

Compute the estimated safe set $\mathcal{X}_t^s$ in \eqref{estimated_action_set} \\

 \textbf{if} the following optimization is feasible: $x(\Tilde{\theta}_t) = \text{argmax}_{x \in \mathcal{X}_t^s} \langle x , \Tilde{\theta}_t \rangle$, \textbf{then}
 
 Set $F=1$, \textbf{else} $F = 0$ \\

\textbf{if} $F = 1$ \textbf{and} $\lambda_{\text{min}} (V_t)\geq \left(\frac{2 L \beta_t}{\kappa_l + \alpha r_{b_l}}\right) ^2$,  \textbf{then}

    Play $x_t = x(\Tilde{\theta}_t)$

\textbf{else}

    Play $x_t = (1-\rho_1) x_{b_t} + \rho_1 \zeta_t$

 Observe reward $y_t$ in \eqref{reward_signal} \\ }\textbf{end for}
\SetAlgoLined
\end{algorithm}

\subsection{Algorithm description}
Let $x_1,\dots,x_t$ be the sequence of the actions and $r_1,\dots,r_t$ be their corresponding rewards. For any $\lambda>0$, we can obtain a regularized least-squares (RLS) estimate $\hat{\theta}_t$ of $\theta_{\star}$ as follows \begin{align}
    \hat{\theta}_t = V_t^{-1} \sum_{s=1}^{t-1} y_s x_s, ~ \text{where} ~ V_t = \lambda I + \sum_{s=1}^{t-1} x_s x_s^{\top}. \label{RLS-estimate}
\end{align}

Algorithm \ref{alg:Safe-Linear-TS1} construct a confidence region \begin{align}
   \mathcal{E}_t(\delta')= \mathcal{E}_t := \{\theta \in \mathbb{R}^d : \| \theta - \hat{\theta}_t \|_{V_t}  \leq \beta_t(\delta')  \}, \label{confidence_region_unkown_reward_parameter}
\end{align} 
where the ellipsoid radius $\beta_t$ is chosen according to the Proposition \ref{abbasi-ellipsoid} in \cite{abbasi2011improved} (restated below for completeness) in order to guarantee that $\theta_{\star} \in \mathcal{E}_t$ with high probability.

\begin{Proposition}\label{abbasi-ellipsoid}
( \cite{abbasi2011improved})
Let Assumptions \ref{ass:sub-gaussian_noise}, \ref{ass:bounded_parameter}, and \ref{ass:actionset} hold. For a fixed $\delta \in (0,1)$, and \begin{align}
    \beta_t(\delta) = R \sqrt{ d \log \left( \frac{1 + \frac{t L^2}{\lambda}}{\delta}  \right)   } + \sqrt{\lambda} S
\end{align} with probability at least $1-\delta$, it holds that $\theta_{\star} \in \mathcal{E}_t$.
\end{Proposition}

\subsubsection{The estimated safe action set}\label{criteria-for-playing-optimistic-action}

Since $\theta_\star$ is unknown to the learner, she does not know whether an action $x \in \mathcal{X}$ is safe or not. Thus, she builds an estimated safe set such that each action $x_t \in \mathcal{X}_t^s$ satisfies the safety constraint for all $v \in \mathcal{E}_t$, i.e., \begin{align}
    \mathcal{X}_t^s  := & \{ x \in \mathcal{X} : \langle x, v \rangle \geq (1-\alpha) r_{b_t}, \forall v \in \mathcal{E}_t \}  = \{ x \in \mathcal{X} : \min_{v \in \mathcal{E}_t} \langle x, v \rangle \geq (1-\alpha) r_{b_t} \} \label{estimated_action_set} \\& = \{ x \in \mathcal{X} :\langle x, \hat{\theta}_t \rangle - \beta_t(\delta') \| x\|_{V_t^{-1}} \geq (1-\alpha) r_{b_t} \}. \label{def:convex-quadratic_estimatedactionset}
\end{align} Note that $\mathcal{X}_t^s$ is easy to compute since \eqref{def:convex-quadratic_estimatedactionset} involves a convex quadratic program. In order to guarantee safety, at each round $t$,  the learner chooses her actions only from this estimated safe set in order to maximize the reward given the sampled parameter $\tilde{\theta}_t$, i.e., \begin{align}
    x(\tilde{\theta}_t) = \arg\max_{x \in \mathcal{X}_t^s} \langle x, \tilde{\theta}_t \rangle,  \label{optimal-safe-action}
\end{align} 
where $\tilde{\theta}_t = \hat{\theta}_t + \beta_t V_t^{-1/2} \eta_t$, and $\eta_t$ is a random IID sample from a distribution $\mathcal{H}^{\text{TS}}$ that  satisfies  certain distributional properties (see \cite{abeille2017linear} or Defn. \ref{distributioanl_propeorty} in Appendix \ref{app:sec:proof_of_theorem_regret_of_sclts} for more details).   
The challenge with $\mathcal{X}_t^s$ is that it contains  actions which are safe with respect to all the parameters in $\mathcal{E}_t$, and not only $\theta_\star$. Hence, there may exist some rounds that $\mathcal{X}_t^s$ is empty. In order to face this problem, the algorithm proceed as follows. At round $t$, if the estimated action set $\mathcal{X}_t^s$ is not empty, SCLTS plays the safe action $x(\tilde{\theta}_t)$ in \eqref{optimal-safe-action} only if the minimum eigenvalue of the Gram matrix $V_t$ is greater than $k_t^1 = \left(\frac{2 L \beta_t}{\kappa_l + \alpha r_{b_l}}\right) ^2$, i.e., $\lambda_{\text{min}}(v_t) \geq k_t^1$, where $k_t^1$ is of order $\mathcal{O}(\log t)$. Otherwise, it plays the conservative action which is presented next.  We show in Appendix \ref{app:sec:proof_of_theorem_regret_of_sclts} that $\lambda_{\text{min}}(v_t) \geq k_t^1$ ensures that for the rounds that SCLTS plays the action $x(\tilde{\theta}_t)$ in \eqref{optimal-safe-action}, the optimal action $x_\star$ belongs to the estimated safe set $\mathcal{X}_t^s$, from which we can bound the regret of Term I in \eqref{regret:decompos}.

\subsubsection{Conservative actions}
In our setting, we assume that the learner is given a baseline policy that at each round $t$ suggests a baseline action $x_{b_t}$.  We employ the idea  proposed in \cite{khezeli2019safe}, which is merging the baseline actions with  random exploration actions under stage-wise safety constraint. In particular,  at each round $t$, SCLTS constructs a conservative action $x_t^{\text{cb}}$ as a convex combination of the baseline action $x_{b_t}$ and a random vector $\zeta_t$ as  follows: \begin{align}
    x_t^{\text{cb}} = (1-\rho_1) x_{b_t} + \rho_1 \zeta_t, \label{conservative-actions}
\end{align} where $\zeta_t$ is assumed to be a sequence of independent, zero-mean and bounded random vectors. Moreover, we assume that $\|\zeta_t \|_2 = 1$ almost surely  and $\sigma^2_{\zeta}=  \lambda_{\text{min}} (\text{Cov}(\zeta_t)) > 0$.  The parameters $\sigma_{\zeta}$ and $\rho_1$ control the exploration level of the conservative actions. In order to ensure that the conservative actions are safe, in Lemma \ref{upper_bound_rho},  we establish an upper bound on $\rho_1$ such that for all $\rho_1 \in (0,\bar{\rho})$, the conservative action $x_t^{\text{cb}} = (1-\rho_1) x_{b_t} + \rho_1 \zeta_t$ is guaranteed to be safe. 

\begin{lemma}\label{upper_bound_rho}
At each round $t$, given the fraction $\alpha$, for any $\rho \in (0,\bar{\rho})$, where $ \bar{\rho} =  \frac{\alpha r_l}{S+r_h}$, the conservative action $x_t^{\text{cb}}=(1-\rho) x_{b_t} + \rho \zeta_t$ is guaranteed to be safe almost surely. 
% \begin{align}
%     \bar{\rho} =  \frac{\alpha r_l}{S+r_h} ,
% \end{align}the conservative action $(1-\rho) x_{b_t} + \rho \zeta_t$ is guaranteed to be safe almost surely.
\end{lemma}
For the ease of notation, in the rest of this paper, we simply assume that $\rho_1 = \frac{r_l}{S+r_h} \alpha$.

At round $t$, SCLTS plays the conservative action $x_t^{\text{cb}}$ if the two conditions defined in Section \ref{criteria-for-playing-optimistic-action} do not hold, i.e., either the estimated safe set $\mathcal{X}_t^s$ is empty or $\lambda_{\text{min}}(V_t) < k_t^1$.

\section{Regret Analysis}\label{sec:regret_analysis}

In this section, we provide a tight regret bound for SCLTS.
In Proposition \ref{decompostion_regret}, we show that the regret of SCLTS can be decomposed into regret caused by choosing safe Thompson Sampling actions plus that of playing conservative actions.  Then,  we bound both terms separately. 
Let $N_{t-1}$ be the set of rounds $i < t$ at which SCLTS plays the action in \eqref{optimal-safe-action}. Similarly, $N_{t-1}^{c} = \{1,\dots,t-1\} - N_{t-1}$ is the set of rounds $j < t$ at which SCLTS plays the conservative actions.

\begin{Proposition}\label{decompostion_regret}
The regret of SCLTS can be decomposed into two terms as follows: \begin{align}
    R(T) \leq \underbrace{\sum_{t \in N_T} \left(\langle x_{\star}, \theta_{\star} \rangle - \langle x_t, \theta_{\star} \rangle \right)}_{\text{Term I}} +  \underbrace{|N_{T}^{{c}}| \left(\kappa_h + \rho_1 (r_h + S) \right)}_{\text{Term II}}\label{regret:decompos}
\end{align}
\end{Proposition}
 The idea of bounding Term I is inspired by \cite{abeille2017linear}: we wish to show that LTS has a constant probability of being "optimistic", in spite of the need to be conservative. In Theorem \ref{thm:bounding_term_I_regret_of_safe_lts}, we provide an upper bound on the regret of Term I which is of order  $\mathcal{O}(d^{3/2} \log^{1/2}d ~ T^{1/2} \log^{3/2}T)$.

\begin{theorem}\label{thm:bounding_term_I_regret_of_safe_lts}
Let $\lambda , L \geq 1$. On event $\{ \theta_\star \in \mathcal{E}_t, \forall t \in [T] \}$, and under  Assumption \ref{ass:lowerbound_reward}, we can bound Term I in \eqref{regret:decompos} as: 
\begin{align}
    \text{Term I} &  \leq ( \beta_T(\delta') + \gamma_T(\delta') (1+\frac{4}{p}) ) \sqrt{2 T d \log{ (1 + \frac{T L^2}{\lambda})}}  + \frac{4 \gamma_T(\delta')}{p} \sqrt{\frac{8T L^2}{\lambda} \log{\frac{4}{\delta}}},
\end{align} where $\delta' = \frac{\delta}{6T}$, and $\gamma_t(\delta) = \beta_t(\delta') \left(  1 + \frac{2}{C} \right) \sqrt{c d \log{(\frac{c' d}{\delta})}}$ 
% \begin{align}
%       \gamma_t(\delta) = \beta_t(\delta') \left(  1 + \frac{2}{C} \right) \sqrt{c d \log{(\frac{c' d}{\delta})}}\,. \label{gamma}
%   \end{align}
\end{theorem}
We note that the regret of Term I has the same bound as that of \cite{abeille2017linear} in spite of the additional safety constraints imposed on the problem. As the next step, in order to bound Term II in \eqref{regret:decompos}, we need to find an upper bound on the number of times $|N_T^c|$ that SCLTS plays the conservative actions up to time $T$. We prove an upper bound on $|N_T^c|$ in Theorem \ref{upperbound:numberofbaseline}.  

% We only provide a proof sketch for Theorem \ref{upperbound:numberofbaseline} due to brevity. The proof requires several lemmas that have been proved in the Appendix \ref{app:theorem_proof:knownbaseline}. 

\begin{theorem}\label{upperbound:numberofbaseline}
Let $\lambda , L \geq 1$. On event $\{ \theta_\star \in \mathcal{E}_t, \forall t \in [T] \}$, and under  Assumption \ref{ass:lowerbound_reward}, it holds that \begin{align}
         | N_{T}^c| \leq  \left( \frac{2L  \beta_{T} }{\rho_1 \sigma_{\zeta}( \kappa_l + \alpha r_l )} \right)^2 + \frac{2h_1^2}{\rho_1^4 \sigma_{\zeta}^4} \log(\frac{d}{\delta}) + \frac{2 Lh_1 \beta_T  \sqrt{8 \log(\frac{d}{\delta})}}{\rho_1^3 \sigma_\zeta^3 (\kappa_l + \alpha r_l )},
\end{align} where $h_1 = 2 \rho_1 (1-\rho_1) L + 2 \rho_1^2 $ and $\rho_1 = (\frac{r_l}{S+r_h}) \alpha$.
\end{theorem}
\begin{remark}\label{remark:upperboundingSCLTS}
The upper bound on the number of times SCLTS plays the conservative actions up to time T provided in Theorem \ref{upperbound:numberofbaseline} has the order $\mathcal{O} \left(  \frac{L^2 d \log(\frac{T}{\delta})  \log(\frac{d}{\delta})    }{\alpha^4  (r_l^2 \wedge r_l^4)\kappa_l ( \sigma_{\zeta}^2 \wedge \sigma_{\zeta}^4) } \right)$.\end{remark}

{ The first idea of the proof is based on the intuition that if a baseline action is played at round $\tau$, then the algorithm does not yet have a good estimate of the unknown parameter $\theta_\star$ and the safe actions played thus far have not yet expanded properly in all directions.
Formally, this translates to small $\lambda_{\text{min}}(V_\tau)$ and the upper bound $O(\log \tau) \geq \lambda_{\text{min}}(V_\tau)$. The second key idea is to exploit the randomized nature of the conservative actions (cf. (11)) to lower bound $\lambda_{\text{min}}(V_\tau)$ by the number of times ($N_{\tau}^c$) that SCLTS plays the baseline actions up to that round (cf. Lemma \ref{lemma:lowerbounding_lambdamin-vt} in the Appendix). Putting these together leads to the advertised upper bound $O(\log T)$ on the total number of times ($N_{T}^c$) the algorithm plays the baseline actions.}

\subsection{Additional Side Constraint with Bandit Feedback}\label{sec:cons_bandit_feedback}

We also consider the setting where the constraint depends on an unknown parameter that is different than the one in reward function. In particular, we assume the constraint of the form \begin{align}
    \langle x_t, \mu_{\star} \rangle \geq (1-\alpha) q_{b_t},  \label{safety_cons_for_unkown_parameter}
\end{align} which needs to be satisfied by the action $x_t$ at every round $t$. In \eqref{safety_cons_for_unkown_parameter}, $\mu_\star$ is a fixed, but unknown and the positive constants $q_{b_t} = \langle x_{b_t}, \mu_\star \rangle$ are known to the learner. In Section \ref{unkown_baseline_reward}, we relax this assumption and we consider the case where the learner does not know the value of $q_{b_t}$. Let $\nu_{b_t} = \langle x_\star, \mu_\star \rangle - \langle x_{b_t}, \mu_\star \rangle$. Similar to Assumption \ref{ass:lowerbound_reward}, we assume there exist  constants $0 \leq \nu_l \leq \nu_h $ and $0< r_l \leq r_h$ such that  $ \nu_l \leq \nu_{b_t} \leq \nu_h $ and $ r_l \leq r_{b_t} \leq r_h$.

We assume that with playing an action $x_t$, the learner observes the following bandit feedback: \begin{align}
    w_t = \langle x_t, \mu_\star \rangle + \chi_t, \label{observ:bandi-feedback}
\end{align} where $\chi_t$ is assumed to be a zero-mean $R$-sub-Gaussian noise. In order to handle this case, we show how SCLTS should be modified, and we propose a new algorithm called SCLTS-BF. The details on SCLTS-BF are presented in Appendix \ref{SCLTS-BF:regret_proof}. In the following, we only mention the  difference of SCLTS-BF with SCLTS, and show an upper bound on its regret.  

The main difference is that SCLTS-BF constructs two confidence regions $\mathcal{E}_t$ in \eqref{confidence_region_unkown_reward_parameter} and $\mathcal{C}_t$ based on the bandit feedback such that $\theta_\star \in \mathcal{E}_t$ and $\mu_\star \in \mathcal{C}_t $ with high probability. Then, based on $\mathcal{C}_t$, it constructs the estimated safe decision set denoted  $\mathcal{P}_t^s = \{ x \in \mathcal{X} : \langle x, v \rangle \geq (1-\alpha) q_{b_t}, \forall v \in \mathcal{C}_t \}$. We  note that SCLTS-BF only plays the  actions from  $\mathcal{P}_t^s$  that are safe with respect to all the parameters in $\mathcal{C}_t$.

We report the details on proving the regret bound for SCLTS-BF in Appendix \ref{SCLTS-BF:regret_proof}. We use the decomposition in Proposition \ref{decompostion_regret}, and we upper bound Term I similar to the Theorem \ref{thm:bounding_term_I_regret_of_safe_lts}. Then, we show an upper bound of order $\mathcal{O} \left(  \frac{L^2 d \log(\frac{T}{\delta})  \log(\frac{d}{\delta})    }{\alpha^4  (q_l^2 \wedge q_l^4)\kappa_l ( \sigma_{\zeta}^2 \wedge \sigma_{\zeta}^4) } \right)$ over  the number of times that SCLTS-BF plays the conservative actions. 

% \begin{theorem}\label{thm:up_bound_SCLTS-bf_number_conservative}
% Let $\lambda , L \geq 1$. On event $\bigg \{ \{ \theta_\star \in \mathcal{E}_t, \forall t \in [T] \} \cap \{ \mu_\star \in \mathcal{C}_t, \forall t \in [T] \} \bigg\} $, and  Assumptions \ref{ass:lowerbound_reward}, we can upper bound the number of times SCLTS-BF plays the conservative actions, i.e., $|N_T^c|$ as: \begin{align}
%   |N_T^c|  \leq  \left( \frac{2L \beta_{T} }{\rho_2 \sigma_{\zeta}(\alpha q_l + \nu_l)} \right)^2 + \frac{2h_2^2}{\rho_2^4 \sigma_{\zeta}^4} \log(\frac{d}{\delta}) + \frac{2L h_2 \beta_T \sqrt{8 \log(\frac{d}{\delta})}}{\rho_2^3 \sigma_\zeta^3 (\alpha q_l + \nu_l)}
% \end{align} where $h_2 = 2 \rho_2 (1-\rho_2) L + 2 \rho_2^2 $ and $\rho_2 = (\frac{q_l}{S+q_h}) \alpha$.
% \end{theorem}

\section{Unknown Baseline Reward }\label{unkown_baseline_reward}

Inspired by \cite{vanroy}, which studies this problem in the presence of {\it safety constraints on the cumulative rewards}, we consider the case where the expected reward of the action chosen by baseline policy, i.e., $r_{b_t}$ is unknown to the learner. However, we assume that the learner knows the value of $r_l$ in \eqref{lowerbound_on_reward_of_baseline_reward}. We describe the required  modifications on SCLTS to handle this case, and present a new algorithm called SCLTS2. Then, we prove the regret bound for SCLTS2, which has the same order as SCLTS. 
% Therefore, the lack of information about the reward of the baseline policy does not cause any harm to our algorithm in terms of the order of the regret. 

% In the following, we rewrite the definition of $\mathcal{X}_t^s$ in \eqref{estimated_action_set} as: \begin{align}
%   \mathcal{X}_t^s =  \{ x \in \mathcal{X} :\min_{v \in \mathcal{E}_t} \langle x, v \rangle \geq (1-\alpha) r_{b_t} = (1-\alpha) \langle x_{b_t}, \theta_\star \rangle \} \label{rewriting_def_safe_estimated_action_set}
% \end{align}

Here, the learner does not know the value of $r_{b_t}$; however, she knows that the unknown parameter $\theta_{\star}$ falls in the confidence region $\mathcal{E}_t$ with high probability. Hence, we can upper bound the RHS of \eqref{cons:safety} with $\max_{ v \in \mathcal{E}_t} \langle x_{b_t}, v \rangle \geq r_{b_t}$. Therefore, any action $x$ that satisfies \begin{align}
    \min_{v \in \mathcal{E}_t} \langle x(\Tilde{\theta}_t), v \rangle \geq (1 - \alpha) \max_{ v \in \mathcal{E}_t} \langle x_{b_t}, v \rangle,
\end{align} is safe with high probability. In order to ensure safety, SCLTS2 only plays the safe actions from the  estimated safe actions set $ \mathcal{Z}_t^s  =  \{ x \in \mathcal{X} :\min_{v \in \mathcal{E}_t} \langle x, v \rangle \geq  (1 - \alpha) \max_{ v \in \mathcal{E}_t} \langle x_{b_t}, v \rangle\}$.  We report the details on SCLTS2 in Appendix \ref{app:sclts2:explantions}.  

Next, we provide an upper bound on the regret of  SCLTS2.  To do so, we first use the decomposition  in Proposition \ref{decompostion_regret}. The regret of Term I is similar to that of SCLTS (Theorem \ref{thm:bounding_term_I_regret_of_safe_lts}), and in Theorem \ref{upperbound:unkownbaseline}, we prove an upper bound on the number of time SCLTS2 plays the conservative actions. Note that similar steps can be generalized to the setting of additional side constraints with bandit feedback.

\begin{theorem}\label{upperbound:unkownbaseline}
Let $\lambda , L \geq 1$. On event $\{ \theta_\star \in \mathcal{E}_t, \forall t \in [T] \}$, and  under Assumption \ref{ass:lowerbound_reward}, we can upper bound the number of times SCLTS2 plays the conservative actions, i.e., $|N_T^c|$ as: \begin{align}
    | N_{T}^c| \leq  \left( \frac{2 L  \beta_{T} (2 - \alpha)}{\rho_3 \sigma_{\zeta}(\kappa_l + \alpha r_l  )} \right)^2   + \frac{2h_3^2}{\rho_3^4 \sigma_{ \zeta}^4} \log(\frac{d}{\delta}) + 
\frac{ 2L h_3 \beta_T (2-\alpha)  }{\rho_3^3 \sigma_{\zeta}^3 (\kappa_l+\alpha r_l )} \sqrt{8 \log(\frac{d}{\delta})},
\end{align} where $h_3 = 2 \rho_3 (1-\rho_3) L + 2 \rho_3^2 $ and $\rho_3 = (\frac{r_l}{S+1}) \alpha$.
\end{theorem}

 \begin{remark}
The regret of SCLTS2 has order of $\mathcal{O} \left(  \frac{L^2 d  \log(\frac{T}{\delta})  \log(\frac{d}{\delta}) (2-\alpha)^2   }{\alpha^4  (r_l^2 \wedge r_l^4)\kappa_l ( \sigma_{\zeta}^2 \wedge \sigma_{\zeta}^4) } \right)$, which has the same rate as that of SCLTS. Therefore, the lack of information about the reward function only hurt the regret with a constant $(2-\alpha)^2$. 
 \end{remark}

 \vspace{-0.25cm}
 \section{Numerical Results}
  \vspace{-0.14cm}
% In this section, we investigate the numerical performance of SCLTS and SCLUCB on synthetic data, and compare it with SEGE algorithm introduced by \cite{khezeli2019safe}. The details on the simulations are deferred to Appendix \ref{app:simulations_details}.

% In Figure \ref{fig:comparison:variance_clts_sege} (left), we plot the average cumulative regret  for  SCLTS and SCLUCB algorithms versus  SEGE algorithms  from \cite{khezeli2019safe}. The shaded regions show  standard deviation around the mean. Figure \ref{fig:comparison:variance_clts_sege} (left) shows the strong dependency of the performance of SEGE algorithm on the specific problem instance. However, the regret of SCLTS and SCLUCB algorithms do not vary significantly under different problem instances.
% Figure \ref{fig:comparison:variance_clts_sege} (middle) shows the average regret of SCLTS for different values of $\alpha$. It shows that for the smaller value of $\alpha$, SCLTS needs to play more conservative actions in order to satisfy safety constraint, and hence suffers a larger regret. 
% Figure \ref{fig:comparison:variance_clts_sege} (right) illustrates the expected reward of SCLTS algorithm in the first $3000$ rounds for $\alpha=0.2$. 

     \begin{figure*}
     \centering
          \includegraphics[width=0.3\textwidth]{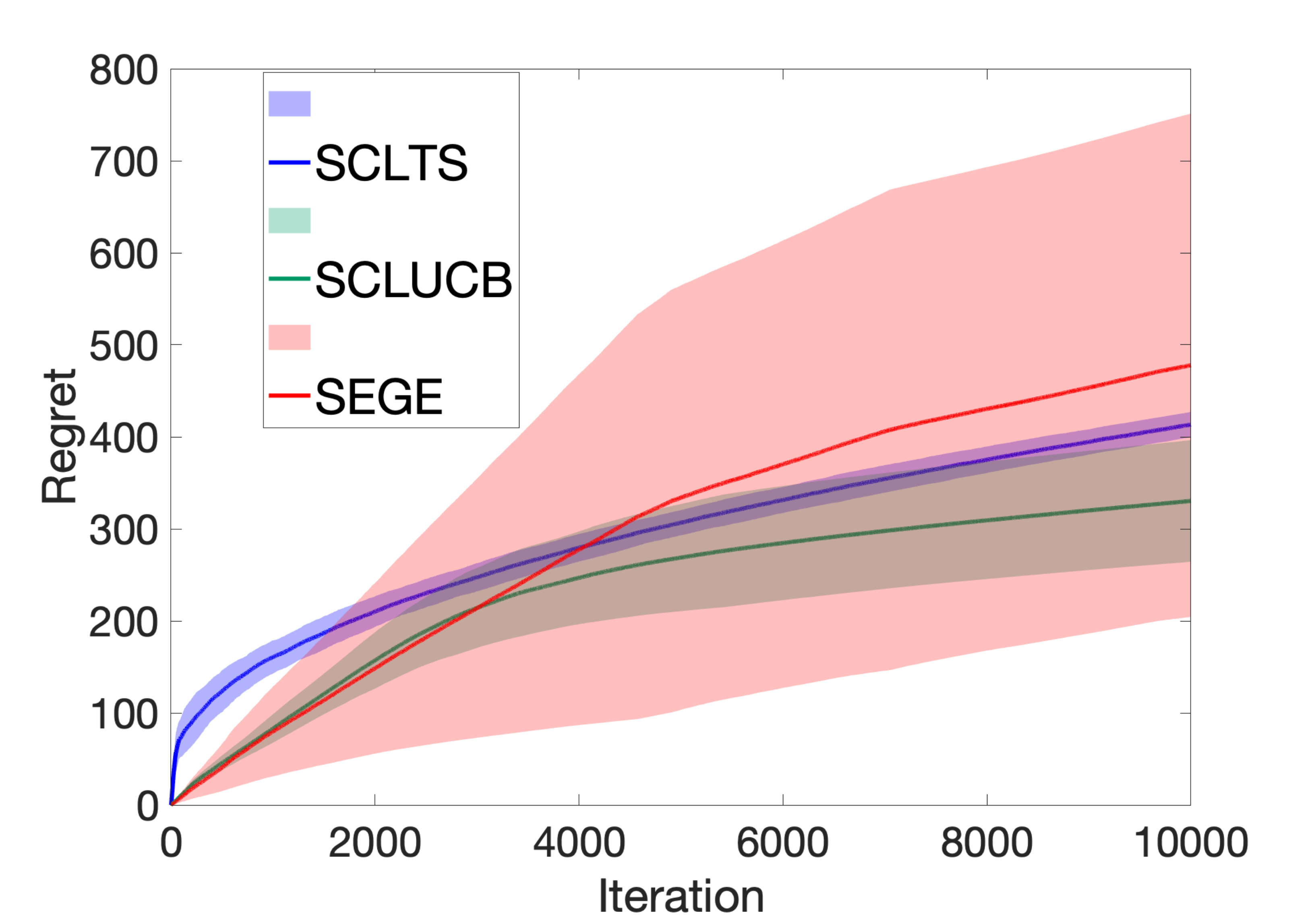} 
          \includegraphics[width=0.3\textwidth]{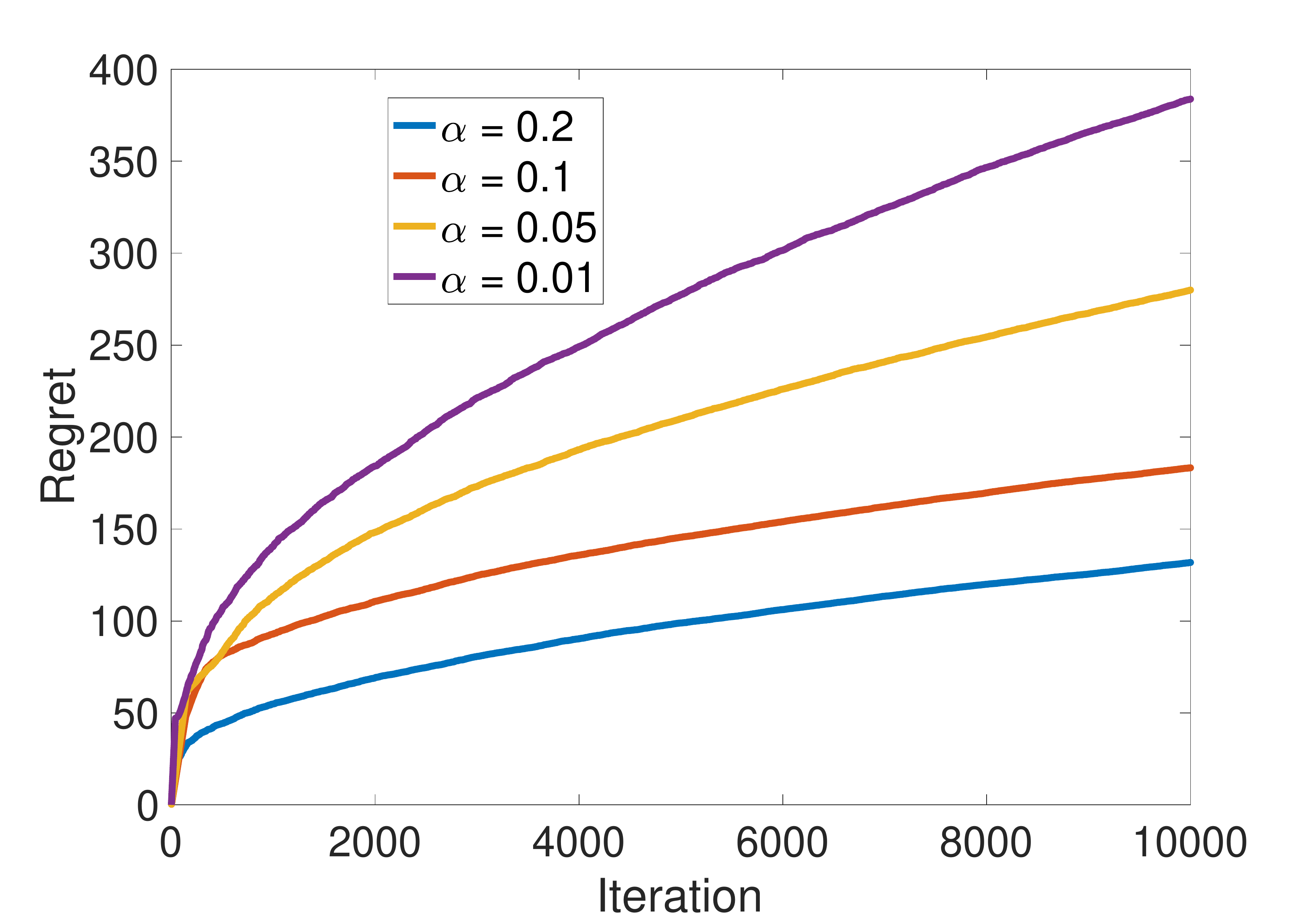}
          \includegraphics[width=0.3\textwidth]{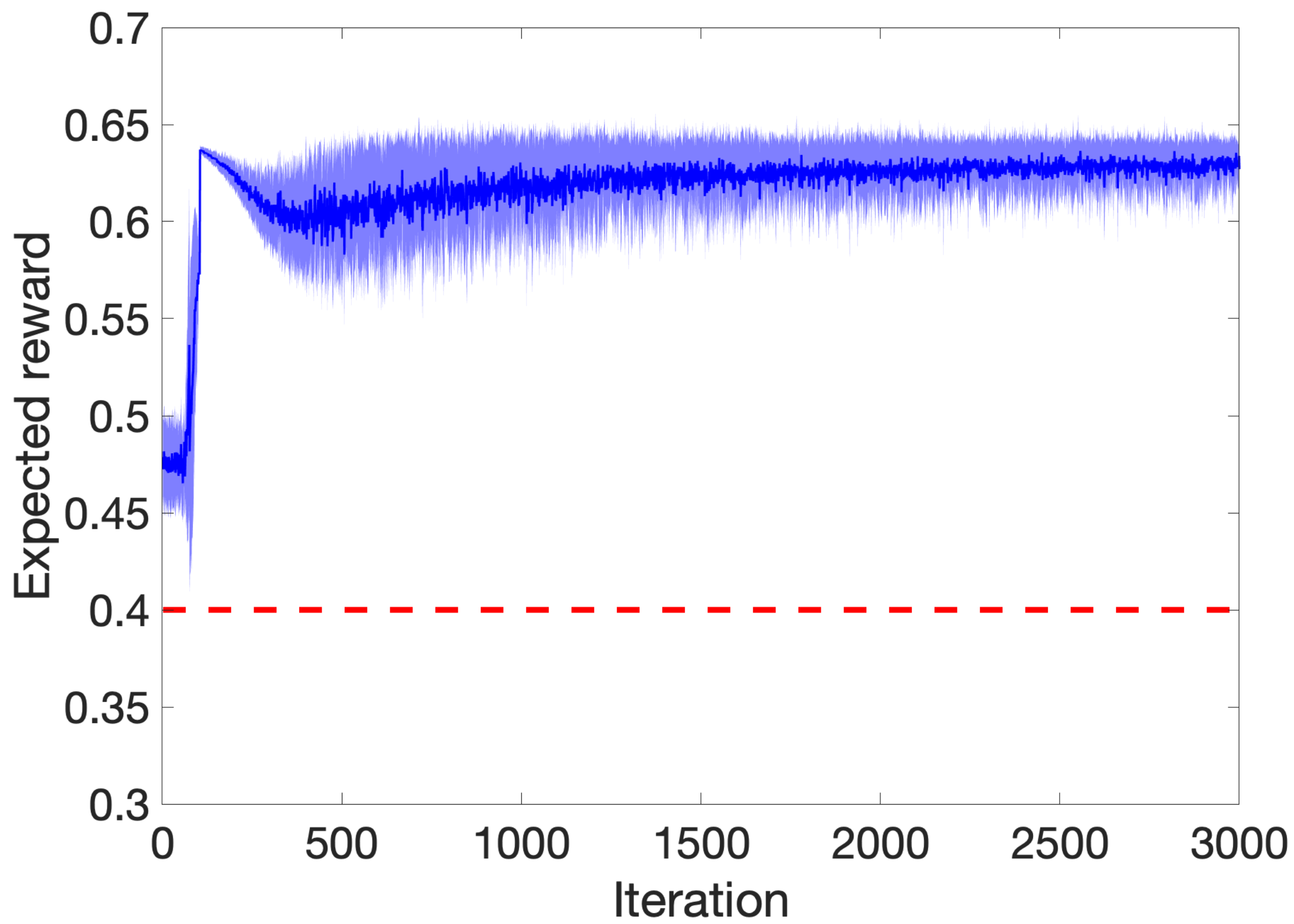}
         \caption{Left: comparison of the  cumulative regret of SCLTS  and SCLUCB  versus SEGE algorithm in \cite{khezeli2019safe}. Middle: average regret (over 100 runs) of SCLTS algorithm  for different values of $\alpha$. Right:  expected reward under SCLTS algorithm in the first $3000$ rounds for $\alpha = 0.2$. }\label{fig:comparison:variance_clts_sege}
   \end{figure*}
In this section,
we investigate the numerical performance of SCLTS and SCLUCB on synthetic data, and compare it with SEGE algorithm introduced by \cite{khezeli2019safe}. In all the implementations, we used the following parameters: $R=0.1, S=1, \lambda = 1, d=2$.  We consider the action set $\mathcal{X}$ to be a unit ball centered on the origin. The reward parameter $\theta_\star$ is drawn from $\mathcal{N}(0, I_4)$. We generate the sequence $\{\zeta_t \}_{t=1}^\infty$ to be IID random vectors that are uniformly distributed on the unit circle. The results are averaged over 100 realizations. 

In Figure \ref{fig:comparison:variance_clts_sege}(left), we plot the cumulative regret of the SCLTS algorithm   and SCLUCB  and SEGE algorithm   from \cite{khezeli2019safe} for $\alpha = 0.2$ over $100$ realizations. The shaded regions show  standard deviation around the mean. In view of the discussion in \cite{Dani08stochasticlinear} regarding computational issues of LUCB algorithms with confidence regions specified with $\ell_2$-norms, we implement a modified version of Safe-LUCB
which uses $\ell_1$-norms instead of $\ell_2$-norms. 
Figure \ref{fig:comparison:variance_clts_sege}(left)  shows that SEGE algorithm suffers a high variance of the regret over different problem instances which shows the strong dependency of the performance of SEGE algorithm on the specific problem instance. However, the regret of SCLTS and SCLUCB algorithms do not vary significantly under different problem instances, and has a low variance. Moreover, the regret of SEGE algorithm grows faster in the beginning steps, since it heavily relies on the baseline action in order to satisfy the safety constraint. However, the randomized nature of SCLTS leads to a natural exploration ability that is much faster in expanding the estimated safe set, and hence it plays the baseline actions less frequently than SEGE algorithm even in the initial exploration stages. 

In Figure \ref{fig:comparison:variance_clts_sege}(middle), we plot the average regret of SCLTS for different values of $\alpha$ over a horizon $T=10000$. Figure \ref{fig:comparison:variance_clts_sege}(middle) shows that,  SCLTS has a better performance (i.e., smaller regret) for the larger value of $\alpha$, since for the small value of $\alpha$, SCLTS needs to be more conservative in order to satisfy the safety constraint, and hence it plays more baseline actions. Moreover, Figure \ref{fig:comparison:variance_clts_sege}(right) illustrates the expected reward of SCLTS algorithm in the first $3000$ rounds. In this setting,  we assume there exists one baseline action $x_b = [0.6,0.5]$, which is available to the learner, $\theta_\star = [0.5,0.4]$ and the safety fraction $\alpha = 0.2$. Thus, the safety threshold is $(1-\alpha) \langle x_b, \theta_\star \rangle = 0.4$ (shown as a dashed red line), which  SCLTS respects in all rounds. In particular,  in initial rounds, SCTLS plays the conservative actions in order to respect the safety constraint, which as shown have an expected reward close to 0.475. Over time as the algorithm achieves a better estimate of the unknown parameter $\theta_\star$, it is able to play more optimistic actions and as such receives higher rewards. 
 \vspace{-0.3cm}
\section{Conclusion}
 \vspace{-0.15cm}
In this paper, we study the stage-wise conservative linear stochastic bandit problem. Specifically, we consider safety constraints that requires the action chosen by the learner at each individual stage to have an expected reward higher than a predefined fraction of the reward of a given baseline policy.  We propose extensions of Linear Thompson Sampling and Linear UCB  in order to minimize the regret of the learner while respecting safety constraint with high probability and provide regret guarantees for them. We also consider the setting of constraints with bandit feedback, where the safety constraint has a different unknown parameter than that of the reward function, and we propose the SCLTS-BF algorithm to handle this case. Third, we study the case where the rewards of the baseline actions are unknown to the learner. Lastly, our numerical experiments compare the performance of our algorithm to SEGE of \cite{khezeli2019safe} and showcase the value of the randomized nature of our exploration phase. In particular, we show that the randomized nature of SCLTS leads to a natural exploration ability that is  faster in expanding the estimated safe set, and hence SCLTS plays the baseline actions less frequently as theoretically shown. For future work, natural extension of the problem setting to generalized linear bandits, and possibly with generalized linear constrains might be of interest. 

 \vspace{-0.3cm}
\section{Acknowledgment}
 \vspace{-0.2cm}
This research is supported by NSF grant 1847096. C. Thrampoulidis was partially supported by the NSF under Grant Number 1934641.
 \vspace{-0.3cm}
\section{Broader Impact}
 \vspace{-0.2cm}
The main goal of this paper is to design and study novel “safe” learning algorithms for safety-critical systems with provable performance guarantees. 
% An example that might benefit from the design of stage-wise conservative learning algorithms arises in societal-scale infrastructure networks such as communication/transportation/data network infrastructure. We focus on the case where the reliability requirements of network operation at each round depends on the reward of the selected action. We consider baseline policy  to  suggest actions that acquire the largest reward known to date, and respect the reliability requirements at each round. In this case, the stage-wise conservative constraint modeled in this paper ensures that at each stage, the reward of action employed by learning algorithm if not better, should be arbitrary close to that of baseline policy, and hence the reliability requirement for network operation must not be violated by the learner.
An example arises in clinical trials where the effect of different therapies on patient's health is not known in advance. We select the baseline actions to be the therapies that have been historically chosen by medical practitioners, and the reward captures the effectiveness of the chosen therapy. The  stage-wise conservative constraint modeled in this paper ensures that at each round the learner should choose a therapy which results in an expected reward if not better, must be  close to the baseline policy. 
Another example arises in societal-scale infrastructure networks such as communication/power/transportation/data network infrastructure.  We focus on the case where the reliability requirements of network operation at each round depends on the reward of the selected  action and certain {\it baseline} actions are known to not violate system constraints and achieve certain levels of operational efficiency as they have been used widely in the past. In this case, the stage-wise conservative constraint modeled in this paper ensures that at each round, the reward of action employed by learning algorithm if not better, should be close to that of baseline policy in terms of network efficiency, and the reliability requirement for network operation must not be violated by the learner. Another example is  in recommender systems that at each round, we wish to avoid recommendations that are extremely disliked by the users. Our proposed stage-wise conservative constrains ensures that at no round would the recommendation system cause severe dissatisfaction for the users (consider perhaps how a really bad personal movie recommendation from a streaming platform would severely affect your view of the said platform).

\bibliographystyle{apalike}
% \bibliography{strings,refs}
\bibliography{bibfile.bib}

\begin{thebibliography}{}

\bibitem[Abbasi-Yadkori et~al., 2011]{abbasi2011improved}
Abbasi-Yadkori, Y., P{\'a}l, D., and Szepesv{\'a}ri, C. (2011).
\newblock Improved algorithms for linear stochastic bandits.
\newblock In {\em Advances in Neural Information Processing Systems}, pages
  2312--2320.

\bibitem[Abeille et~al., 2017]{abeille2017linear}
Abeille, M., Lazaric, A., et~al. (2017).
\newblock Linear thompson sampling revisited.
\newblock {\em Electronic Journal of Statistics}, 11(2):5165--5197.

\bibitem[Agrawal and Goyal, 2013]{agrawal2013thompson}
Agrawal, S. and Goyal, N. (2013).
\newblock Thompson sampling for contextual bandits with linear payoffs.
\newblock In {\em International Conference on Machine Learning}, pages
  127--135.

\bibitem[{Akametalu} et~al., 2014]{7039601}
{Akametalu}, A.~K., {Fisac}, J.~F., {Gillula}, J.~H., {Kaynama}, S.,
  {Zeilinger}, M.~N., and {Tomlin}, C.~J. (2014).
\newblock Reachability-based safe learning with gaussian processes.
\newblock In {\em 53rd IEEE Conference on Decision and Control}, pages
  1424--1431.

\bibitem[Amani et~al., 2019]{amani2019linear}
Amani, S., Alizadeh, M., and Thrampoulidis, C. (2019).
\newblock Linear stochastic bandits under safety constraints.
\newblock In {\em Advances in Neural Information Processing Systems}, pages
  9252--9262.

\bibitem[Aswani et~al., 2013]{aswani2013provably}
Aswani, A., Gonzalez, H., Sastry, S.~S., and Tomlin, C. (2013).
\newblock Provably safe and robust learning-based model predictive control.
\newblock {\em Automatica}, 49(5):1216--1226.

\bibitem[Auer et~al., 2002]{Auer}
Auer, P., Cesa-Bianchi, N., and Fischer, P. (2002).
\newblock Finite-time analysis of the multiarmed bandit problem.
\newblock {\em Mach. Learn.}, 47(2-3):235--256.

\bibitem[Bubeck and Eldan, 2016]{bubeck2016multi}
Bubeck, S. and Eldan, R. (2016).
\newblock Multi-scale exploration of convex functions and bandit convex
  optimization.
\newblock In {\em Conference on Learning Theory}, pages 583--589.

\bibitem[Dani et~al., 2008]{Dani08stochasticlinear}
Dani, V., Hayes, T.~P., and Kakade, S.~M. (2008).
\newblock Stochastic linear optimization under bandit feedback.

\bibitem[Filippi et~al., 2010]{filippi2010parametric}
Filippi, S., Cappe, O., Garivier, A., and Szepesv{\'a}ri, C. (2010).
\newblock Parametric bandits: The generalized linear case.
\newblock In {\em Advances in Neural Information Processing Systems}, pages
  586--594.

\bibitem[Katariya et~al., 2018]{katariya2018conservative}
Katariya, S., Kveton, B., Wen, Z., and Potluru, V.~K. (2018).
\newblock Conservative exploration using interleaving.
\newblock {\em arXiv preprint arXiv:1806.00892}.

\bibitem[Kaufmann et~al., 2012]{kaufmann2012thompson}
Kaufmann, E., Korda, N., and Munos, R. (2012).
\newblock Thompson sampling: An asymptotically optimal finite-time analysis.
\newblock In {\em International Conference on Algorithmic Learning Theory},
  pages 199--213. Springer.

\bibitem[Kazerouni et~al., 2017]{vanroy}
Kazerouni, A., Ghavamzadeh, M., Abbasi, Y., and Van~Roy, B. (2017).
\newblock Conservative contextual linear bandits.
\newblock In Guyon, I., Luxburg, U.~V., Bengio, S., Wallach, H., Fergus, R.,
  Vishwanathan, S., and Garnett, R., editors, {\em Advances in Neural
  Information Processing Systems 30}, pages 3910--3919. Curran Associates, Inc.

\bibitem[Khezeli and Bitar, 2019]{khezeli2019safe}
Khezeli, K. and Bitar, E. (2019).
\newblock Safe linear stochastic bandits.
\newblock {\em arXiv preprint arXiv:1911.09501}.

\bibitem[Koller et~al., 2018]{koller2018learning}
Koller, T., Berkenkamp, F., Turchetta, M., and Krause, A. (2018).
\newblock Learning-based model predictive control for safe exploration.
\newblock In {\em 2018 IEEE Conference on Decision and Control (CDC)}, pages
  6059--6066. IEEE.

\bibitem[Li et~al., 2017]{li2017provably}
Li, L., Lu, Y., and Zhou, D. (2017).
\newblock Provably optimal algorithms for generalized linear contextual
  bandits.
\newblock In {\em Proceedings of the 34th International Conference on Machine
  Learning-Volume 70}, pages 2071--2080. JMLR. org.

\bibitem[Mansour et~al., 2015]{mansour2015bayesian}
Mansour, Y., Slivkins, A., and Syrgkanis, V. (2015).
\newblock Bayesian incentive-compatible bandit exploration.
\newblock In {\em Proceedings of the Sixteenth ACM Conference on Economics and
  Computation}, pages 565--582.

\bibitem[{Moradipari} et~al., 2020]{9053865}
{Moradipari}, A., {Alizadeh}, M., and {Thrampoulidis}, C. (2020).
\newblock Linear thompson sampling under unknown linear constraints.
\newblock In {\em ICASSP 2020 - 2020 IEEE International Conference on
  Acoustics, Speech and Signal Processing (ICASSP)}, pages 3392--3396.

\bibitem[Moradipari et~al., 2019]{moradipari2019safe}
Moradipari, A., Amani, S., Alizadeh, M., and Thrampoulidis, C. (2019).
\newblock Safe linear thompson sampling with side information.
\newblock {\em arXiv}, page arXiv:1911.02156.

\bibitem[Moradipari et~al., 2018]{moradipari2018learning}
Moradipari, A., Silva, C., and Alizadeh, M. (2018).
\newblock Learning to dynamically price electricity demand based on multi-armed
  bandits.
\newblock In {\em 2018 IEEE Global Conference on Signal and Information
  Processing (GlobalSIP)}, pages 917--921. IEEE.

\bibitem[Ostafew et~al., 2016]{ostafew2016robust}
Ostafew, C.~J., Schoellig, A.~P., and Barfoot, T.~D. (2016).
\newblock Robust constrained learning-based nmpc enabling reliable mobile robot
  path tracking.
\newblock {\em The International Journal of Robotics Research},
  35(13):1547--1563.

\bibitem[Rusmevichientong and Tsitsiklis, 2010]{Tsitsiklis}
Rusmevichientong, P. and Tsitsiklis, J.~N. (2010).
\newblock Linearly parameterized bandits.
\newblock {\em Mathematics of Operations Research}, 35(2):395--411.

\bibitem[Russo and Van~Roy, 2016]{russo2016information}
Russo, D. and Van~Roy, B. (2016).
\newblock An information-theoretic analysis of thompson sampling.
\newblock {\em The Journal of Machine Learning Research}, 17(1):2442--2471.

\bibitem[Sui et~al., 2018]{sui2018stagewise}
Sui, Y., Burdick, J., Yue, Y., et~al. (2018).
\newblock Stagewise safe bayesian optimization with gaussian processes.
\newblock In {\em International Conference on Machine Learning}, pages
  4788--4796.

\bibitem[Sui et~al., 2015]{Krause}
Sui, Y., Gotovos, A., Burdick, J.~W., and Krause, A. (2015).
\newblock Safe exploration for optimization with gaussian processes.
\newblock In {\em Proceedings of the 32Nd International Conference on
  International Conference on Machine Learning - Volume 37}, ICML'15, pages
  997--1005. JMLR.org.

\bibitem[Thompson, 1933]{thompson1933likelihood}
Thompson, W.~R. (1933).
\newblock On the likelihood that one unknown probability exceeds another in
  view of the evidence of two samples.
\newblock {\em Biometrika}, 25(3/4):285--294.

\bibitem[Tropp, 2012]{tropp2012user}
Tropp, J.~A. (2012).
\newblock User-friendly tail bounds for sums of random matrices.
\newblock {\em Foundations of computational mathematics}, 12(4):389--434.

\bibitem[{Tucker} et~al., 2020]{9125942}
{Tucker}, N., {Moradipari}, A., and {Alizadeh}, M. (2020).
\newblock Constrained thompson sampling for real-time electricity pricing with
  grid reliability constraints.
\newblock {\em IEEE Transactions on Smart Grid}, pages 1--1.

\bibitem[Wu et~al., 2016]{wu2016conservative}
Wu, Y., Shariff, R., Lattimore, T., and Szepesv{\'a}ri, C. (2016).
\newblock Conservative bandits.
\newblock In {\em International Conference on Machine Learning}, pages
  1254--1262.

\end{thebibliography}

\newpage
\appendix

\section{Proof of Proposition \ref{decompostion_regret}}

From the definition of regret, we can write \begin{align}
    R(T) &= \sum_{t \in N_T} \left(\langle x_{\star}, \theta_{\star} \rangle - \langle x_t, \theta_{\star} \rangle \right) + \sum_{t \in N_T^c}  \left(\langle x_{\star}, \theta_{\star} \rangle - \langle (1-\rho_1)x_{b_t} - \rho_1 \zeta_t, \theta_{\star} \rangle \right) \nonumber \\& =  \sum_{t \in N_T} \left(\langle x_{\star}, \theta_{\star} \rangle - \langle x_t, \theta_{\star} \rangle \right) + \sum_{t \in N_T^c} \bigg(\langle x_{\star}, \theta_{\star} \rangle - \langle x_{b_t}, \theta_\star \rangle + \rho_1 \langle x_{b_t}, \theta_\star \rangle + \rho_1 \langle \zeta_t , \theta_\star \rangle \bigg)
    \nonumber\\& \leq \sum_{t \in N_T} \left(\langle x_{\star}, \theta_{\star} \rangle - \langle x_t, \theta_{\star} \rangle \right)+ |N_T^c| \left(\kappa_h + \rho_1 (r_h + S) \right). 
\end{align}

\section{Proof of Lemma \ref{upper_bound_rho}}

In order to ensure that the  conservative action $x_t = (1-\rho) x_{b_t} - \rho \zeta_t$ is safe, we need to show that it satisfies \eqref{cons:safety}. Hence, it suffices to  show that \begin{align}
    \langle (1-\rho) x_{b_t} - \rho \zeta_t, \theta_{\star} \rangle \geq (1-\alpha) r_{b_t}. \label{showing_lower_bound_for_conservative_actions}
\end{align} We can lower bound the LHS of \eqref{showing_lower_bound_for_conservative_actions} as follows: \begin{align}
    \langle (1-\rho) x_{b_t} - \rho \zeta_t, \theta_{\star} \rangle = r_{b_t} - \rho r_{b_t} - \rho \langle  \zeta_t, \theta_{\star} \rangle \geq r_{b_t} - \rho r_{b_t} - \rho S. \nonumber
\end{align} Recall that $\|\zeta_t \|_2 = 1$ almost surely, and due to Assumption \ref{ass:bounded_parameter}, we know that  $\|\theta_{\star}\|_2 \leq S$. Hence,  it suffices to show that \begin{align}
r_{b_t} - \rho r_{b_t} - \rho S \geq (1-\alpha) r_{b_t}, \nonumber
\end{align} or equivalently, 
 \begin{align}
 \rho r_{b_t} + \rho S \leq \alpha r_{b_t} \label{baseline_noise_bound}
\end{align}   From \eqref{baseline_noise_bound} we can write \begin{align}
    \rho \leq \frac{\alpha r_{b_t}}{S + r_{b_t}}. \label{real_upper_bound_on_rho}
\end{align} Therefore, for any $\rho $ satisfying \eqref{real_upper_bound_on_rho}, the conservative action $x_t = (1-\rho) x_{b_t} + \rho \zeta_t$ is guaranteed to be safe almost surely. Then, we lower bound the right hand side of \eqref{real_upper_bound_on_rho} using Assumption \ref{ass:lowerbound_reward}, and we establish the following upper bound on $\rho$, \begin{align}
    \rho \leq \frac{\alpha r_{l}}{S + r_h}. \label{final_upper_bound_on_rho}
\end{align} Therefore, for any $\rho \in (0,\bar{\rho})$, where 
    $\bar{\rho} = \frac{\alpha r_{l}}{S + r_{h}}$, the conservative actions are safe.

\section{Proof of Theorem \ref{thm:bounding_term_I_regret_of_safe_lts}} \label{app:sec:proof_of_theorem_regret_of_sclts}
In this section, we provide an upper bound on the regret of Term I in \eqref{regret:decompos}. We first rewrite   Term I as follows: \begin{align}
    \sum_{t \in N_T} \left(\langle x_{\star}, \theta_{\star} \rangle - \langle x_t, \theta_{\star} \rangle \right) \label{apndx:prf_thm_term_I}
\end{align}
Clearly, it would be beneficial to show that \eqref{apndx:prf_thm_term_I} is non-positive. However, as stated in \cite{abeille2017linear} (in the case of linear TS applied to the standard stochastic linear bandit problem with no safety constraints), this cannot be the case in general. Instead, to bound regret in the unconstrained case, \cite{abeille2017linear}  argues that it suffices to show that \eqref{apndx:prf_thm_term_I} is non-positive with a constant probability. But what happens in  the safety-constrained scenario? It turns out that once the above stated event happens with constant probability  (in our case, in the presence of safety constraints), the rest of the argument by \cite{abeille2017linear} remains unaltered. Therefore, our main contribution in the proof of Theorem \ref{thm:bounding_term_I_regret_of_safe_lts} is to show that \eqref{apndx:prf_thm_term_I} is non-positive with a constant probability in spite of the limitations on actions imposed because of the safety constraints. To do so, let \begin{align}
    \Theta_t^{\text{opt}} = \{ \theta \in \mathbb{R}^d : \langle x(\theta), \theta \rangle \geq \langle x_\star, \theta_\star \rangle \},
\end{align} be the so-called \textit{set of optimistic parameters}, where $x(\tilde{\theta}_t) = \arg\max_{x \in \mathcal{X}_t^s} \langle x, \tilde{\theta}_t \rangle$  is the optimal safe action for the sampled parameter $\tilde{\theta}_t$ chosen from the estimated safe action set $\mathcal{X}_t^s$.  LTS is considered optimistic at round $t$, if it samples the parameter $\tilde{\theta}_t$ from the set of optimistic parameters $\Theta_t^{\text{opt}}$ and plays the action $x(\tilde{\theta}_t)$.  In Lemma \ref{lemma:optimistic}, we show that SCLTS is optimistic with  constant probability despite the safety constraints. Before that, let us restate the distributional properties put forth in \cite{abeille2017linear} for the noise $\eta \sim \mathcal{H}^{\text{TS}}$ that are required to ensure the right balance of exploration and exploitation.
\begin{definition} ( Definition 1. in \cite{abeille2017linear}) \label{distributioanl_propeorty}
$\mathcal{H}^{\text{TS}}$ is a multivariate distribution on $\mathbb{R}^d$  absolutely continuous with respect to the Lebesgue measure which satisfies the following properties: \begin{itemize}
    \item (anti-concentration) there exists a strictly positive probability $p$ such that for any $u \in \mathbb{R}^d$ with $\| u \|_2 = 1$, \begin{align}
        \mathbb{P}_{\eta \sim \mathcal{H}^{\text{TS}} }\left(  \langle u, \eta     \rangle \geq 1 \right) \geq p.
    \end{align}
    \item (concentration) there exists positive constants $c,c'$ such that $\forall \delta \in (0,1)$ \begin{align}
        \mathbb{P}_{\eta \sim \mathcal{H}^{\text{TS}} }\left(  \|\eta\| \leq \sqrt{cd \log(\frac{c'd}{\delta})} \right) \geq 1-\delta.
    \end{align}
\end{itemize}
\end{definition}
\begin{lemma}\label{lemma:optimistic}
Let $\Theta_t^{\text{opt}} = \{ \theta \in \mathbb{R}^d : \langle x(\theta), \theta \rangle \geq \langle x_\star, \theta_\star \rangle \}$ be the set of the optimistic parameters. For round $t\in N_T$, SCLTS  samples the optimistic parameter $\tilde{\theta}_t \in \Theta_t^{\text{opt}}$ and plays the corresponding safe action  $x(\tilde{\theta}_t)$ frequently enough, i.e., \begin{align}
    \mathbb{P}(\tilde{\theta}_t \in \Theta_t^{\text{opt}} ) \geq p.
\end{align}
\end{lemma}

\begin{proof}
We need to show that for rounds $t \in N_T$ \begin{align}
    \mathbb{P} \left(  \langle x(\tilde{\theta}_t) , \tilde{\theta}_t \rangle \geq \langle x_\star, \theta_\star \rangle    \right)\geq p. \label{optimsic_ineqaulity}
\end{align}
First, we show that for rounds $t \in N_T$, $x_\star$ falls in the estimated safe set, i.e., $x_\star \in \mathcal{X}_t^s$. 
To do so, we need to show that 
\begin{align}
    \langle x_\star , \hat{\theta}_t \rangle - \beta_t \| x_\star \|_{V_t^{-1}}  \geq (1-\alpha) r_{b_t}, 
\end{align} using $\| \theta_\star - \hat{\theta}_t \|_{V_t} \leq \beta_t$,  it suffices that \begin{align}
     \langle x_\star, \theta_\star \rangle - 2 \beta_t \| x_\star\|_{V_t^{-1}} \geq (1-\alpha) r_{b_t}. \label{applying_shrunk_safe_set}
\end{align}
But we know  that $\| x_\star \|_{V_t^{-1}} \leq \frac{\|x_\star\|_2}{\sqrt{\lambda_{\text{min}}(V_t)}} \leq \frac{L}{\sqrt{\lambda_{\text{min}}(V_t)}}$, where we also used  Assumption \ref{ass:actionset} to bound $\| x_\star \|_2$. Hence, we can get
\begin{align}
  \langle x_\star, \theta_\star \rangle - 2 \beta_t \| x_\star\|_{V_t^{-1}} \geq    \langle x_\star, \theta_\star \rangle - \frac{2 \beta_t L}{\sqrt{\lambda_{\text{min}}(V_t)}}. \label{finding_lower_bound_on_shrunk_set}
\end{align}
By substituting \eqref{finding_lower_bound_on_shrunk_set} in \eqref{applying_shrunk_safe_set},  it suffices to show that \begin{align}
  \kappa_{b_t}  + \alpha r_{b_t} \geq \frac{2 \beta_t L}{\sqrt{\lambda_{\text{min}}(V_t)}}, \label{to_show_x*_in_xts}
\end{align} or equivalently, \begin{align}
    \lambda_{\text{min}}(V_t) \geq \left(\frac{2 L \beta_t}{\kappa_t + \alpha r_{b_t}}\right) ^2. \label{suffices_to_whow_x*_in_safe_set_estimated}
\end{align} To show \eqref{suffices_to_whow_x*_in_safe_set_estimated}, simply recall that $\lambda_{\text{min}}(V_t) \geq k_t^1$, where $k^1_t = \left(\frac{2 L \beta_t}{\kappa_l + \alpha r_l}\right) ^2$. 
Therefore, $x_\star \in \mathcal{X}_t^s$ for $t \in N_T$. Note that we are not interested in expanding the safe set in all possible directions. Instead, what aligns with the objective of minimizing regret, is expanding the safe set in the “correct” direction, that of $x_\star$. Therefore, $\lambda_{\text{min}}(V_t) \geq \mathcal{O}(\log t)$ provides enough  expansion of the safe set to  bound the Term I in \eqref{regret:decompos}.

The rest of the proof is similar as in \cite[Lemma 3]{abeille2017linear};  we include in here for completeness. 

For rounds $t \in N_T$, we know that \begin{align}
\langle x(\Tilde{\theta}_t), \Tilde{\theta}_t \rangle \geq \langle x_\star, \Tilde{\theta}_t \rangle, \nonumber
\end{align}  since $x(\tilde{\theta}_t) = \arg\max_{x \in \mathcal{X}_t^s} \langle x, \tilde{\theta}_t \rangle$ and we have already shown that $x_\star \in \mathcal{X}_t^s$. Therefore, it suffices to show that 
\begin{align}
    \mathbb{P}\left( \langle x_\star, \Tilde{\theta}_t \rangle \geq \langle x_\star , \theta_\star \rangle \right) \geq p.  \label{showing_optimistic_for_xstar}
\end{align} From the definition of $\tilde{\theta}_t$, we can rewrite \eqref{showing_optimistic_for_xstar} as \begin{align}
    \mathbb{P}\left( \langle x_\star, \hat{\theta}_t \rangle + \beta_t \langle x_\star, V_t^{-1/2} \eta_t \rangle \geq \langle x_\star , \theta_\star \rangle \right) \geq p, \nonumber
\end{align} or equivalently, \begin{align}
    \mathbb{P}\left( \beta_t \langle x_\star, V_t^{-1/2} \eta_t \rangle \geq \langle x_\star , \theta_\star -  \hat{\theta}_t \rangle \right) \geq p. \label{to_use_caucy_for_optimistic_for_xstrar}
\end{align} Then, we use Cauchy-Schwarz for the LHS of \eqref{to_use_caucy_for_optimistic_for_xstrar}, and given the fact that $\| \theta_\star -  \hat{\theta}_t\|_{V_t} \leq \beta_t$, we get \begin{align}
    \mathbb{P}\left( \langle x_\star, V_t^{-1/2} \eta_t \rangle \geq \| x_\star \|_{V_t^{-1/2}} \right) \geq p, \nonumber
\end{align} or equivalently, \begin{align}
    \mathbb{P}\left( \langle u_t, \eta_t \rangle \geq  1\right) \geq p, \label{anti-concen-from-abeile}
\end{align} where $u_t = \frac{x_\star V_t^{-1/2}}{\| x_\star \|_{V_t^{-1/2}}}$. Therefore, $\|u_t \|_2 = 1$ by construction. At last, we know that \eqref{anti-concen-from-abeile} is true thanks to the anti-concentration distributional property of the parameter $\eta_t$ in Definition \ref{distributioanl_propeorty}.
\end{proof}
As mentioned, after showing that SCLTS for rounds $t \in N_T $ samples from the set of optimistic parameters with a constant probability, the rest of the proof for bounding the regret of Term I is similar to that of \cite{abeille2017linear}. In particular, we conclude with the following bound 
\begin{align}
\text{Term I}:= & \sum_{t \in N_T} \left(\langle x_{\star}, \theta_{\star} \rangle - \langle x_t, \theta_{\star} \rangle \right) \nonumber \\& ( \beta_T(\delta') + \gamma_T(\delta') (1+\frac{4}{p}) ) \sqrt{2 |N_T| d \log{ (1 + \frac{|N_T| L^2}{\lambda})}}  + \frac{4 \gamma_T(\delta')}{p} \sqrt{\frac{8 |N_T| L^2}{\lambda} \log{\frac{4}{\delta}}},
\end{align} where $\delta' = \frac{\delta}{6 |N_T|}$, and, \begin{align}
      \gamma_t(\delta) = \beta_t(\delta') \left(  1 + \frac{2}{C} \right) \sqrt{c d \log{(\frac{c' d}{\delta})}}, 
  \end{align} and since $N_T \leq T$, the proof is completed.

\section{Proof of Theorem \ref{upperbound:numberofbaseline}}\label{proof_of_number_bound_theorem}
In this section,
we prove an upper bound of order $\mathcal{O}(\log T)$ on the number of times that SCLTS plays the conservative actions. 

Let $\tau$ be any round that the algorithm plays the conservative action, i.e., at round $\tau$, either $F=0$ or $\lambda_{\text{min}} (V_\tau) < k_\tau^1 = \left(\frac{2 L \beta_\tau}{\kappa_\tau + \alpha r_{b_\tau}}\right) ^2$.
By definition, if $F=0$, we have 
 \begin{align}
    \nexists x \in \mathcal{X} : \langle x , \hat{\theta}_\tau \rangle - \beta_\tau \| x \|_{V_\tau^{-1}}   \geq (1-\alpha) r_{b_\tau},
\end{align} and since we know that $x_\star \in \mathcal{X}$, and $\theta_\star \in \mathcal{E}_t$ with high probability, we can write  \begin{align}
 \langle x_\star , \theta_\star \rangle - 2\beta_\tau \| x_\star \|_{V_\tau^{-1}}  \leq \langle x_\star , \hat{\theta}_\tau \rangle - \beta_\tau \| x_\star \|_{V_\tau^{-1}}  < (1-\alpha) r_{b_\tau}. \label{boundingconstraint_value_IN_DIFFERENT_PARAMETER}
\end{align} From \eqref{boundingconstraint_value_IN_DIFFERENT_PARAMETER}, we can get \begin{align}
    \kappa_{b_\tau} + \alpha r_{b_\tau} < 2 \beta_\tau \| x_\star \|_{V_\tau^{-1}} \leq \frac{2 \beta_\tau L}{\sqrt{\lambda_{\text{min}}(V_\tau)}},
\end{align} and hence the following upper bound on minimum eigenvalue of the Gram matrix: \begin{align}
    \lambda_{\text{min}}(V_\tau) < \left(\frac{2 \beta_\tau L}{\kappa_{b_\tau} + \alpha r_{b_\tau}}\right)^2 \leq k^1_\tau. \label{eq_for_puttinginsidethepapaer}
\end{align}

Therefore, at any round $\tau$ that a conservative action is played, whether it is because $F=0$, or because we have $\{ \lambda_{\text{min}}(V_\tau) < k_\tau \}$, we can always conclude that   \begin{align}
\lambda_{\text{min}}(V_\tau) < k^1_\tau. \label{upperbounding_lambamin_when_conservarive-happens}
\end{align}

The remainder of the proof builds on two auxiliary lemmas. First,
 in  Lemma \ref{lemma:lowerbounding_lambdamin-vt}, we show that the minimum eigenvalue of the Gram matrix $V_t$ is lower bounded with the number of times SCLTS plays the conservative actions.

\begin{lemma}\label{lemma:lowerbounding_lambdamin-vt}
Under Assumptions \ref{ass:sub-gaussian_noise}, \ref{ass:bounded_parameter}, and \ref{ass:actionset}, it holds that \begin{align}
    \mathbb{P}(  \lambda_{\text{min}} (V_{t}) \leq t ) \leq d \exp{ \left( - \frac{(\rho_1^2 |N_t^c| \sigma_{\zeta}^2 - t)^2}{8 |N_t^c| h_1^2 } \right)},
\end{align} where $h_1 = 2 \rho_1 (1-\rho_1) L + 2 \rho_1^2 $ and $\rho_1 = (\frac{r_l}{S+r_h}) \alpha$.
\end{lemma}

Using \eqref{upperbounding_lambamin_when_conservarive-happens} and applying Lemma \ref{lemma:lowerbounding_lambdamin-vt}, it can be checked that with probability $1-\delta$, \begin{align}
    \left( \frac{2L \beta_{\tau}  }{\kappa_l +\alpha r_l } \right)^2 > \rho_1^2 |N_{\tau}^c| \sigma_{\zeta}^2  - \sqrt{8 |N_{\tau}^c| h_1^2 \log(\frac{d}{\delta})}.
\end{align}

This gives an explicit inequality that must be satisfied by $\tau$. Solving with respect to $\tau$ leads to the desired. In particular, we apply simple Lemma \ref{algebra_for_upperbound} below. 

\begin{lemma}\label{algebra_for_upperbound}
For any $a,b,c > 0$, if $ax - \sqrt{bx} < c$, then the following holds for $x \geq 0$ \begin{align} 
  0 \leq x < \frac{2ac + b + \sqrt{b^2 + 4abc}}{2a^2}.
\end{align}
\end{lemma}

Using Lemma \ref{algebra_for_upperbound} results in the following upper bound on the $|N_{\tau}^c|$ \begin{align}
    | N_{\tau}^c| \leq  \left( \frac{2L  \beta_{\tau} }{\rho_1 \sigma_{\zeta}(\kappa_l + \alpha r_l )} \right)^2 + \frac{2h_1^2}{\rho_1 ^4 \sigma_{\zeta}^4} \log(\frac{d}{\delta}) + \frac{h_1  2L  \beta_{\tau} }{(\kappa_l +\alpha r_l)\rho_1^3 \sigma_{\zeta}^3} \sqrt{8 \log(\frac{d}{\delta})}.
\end{align} Therefore, we can upper bound $N_T^c$ with the following: \begin{align}
    | N_{T}^c| \leq  \left( \frac{2L   \beta_{T} }{\rho_1 \sigma_{\zeta}( \kappa_l +\alpha r_l )} \right)^2 + \frac{2h_1^2}{\rho_1^4 \sigma_{\zeta}^4} \log(\frac{d}{\delta}) + \frac{2L h_1 \beta_T \sqrt{8 \log(\frac{d}{\delta})}}{\rho_1^3 \sigma^3 (\kappa_l +\alpha r_l)},
\end{align} which has  order $\mathcal{O} \left(  \frac{L^2 d \log(\frac{T}{\delta})}{\alpha^2 r_l^2 (\kappa_l + \alpha r_l )^2 \sigma_{\zeta}^2 } + \bigg( \frac{L^2}{\alpha^2 r_l^2 \sigma_{\zeta}^4} + d^2  \bigg) \log(\frac{d}{\delta}) \right)$, as promised.

\subsection{ Proof of Lemma \ref{lemma:lowerbounding_lambdamin-vt}}

Our objective is to establish a lower bound on  $\lambda_{\text{min}} (V_{t})$ for all $t$. It holds that \begin{align}
    V_t &= \lambda I + \sum_{s=1}^t x_s x_s^{\top} \nonumber \\& \succeq \sum_{s \in N_t^c} \left( (1-\rho_1) x_{b_s} - \rho_1 \zeta_s \right) \left( (1-\rho_1) x_{b_s} - \rho_1 \zeta_s \right)^{\top} \nonumber\\&
    = \sum_{s \in N_t^c} \bigg( (1-\rho_1)^2 x_{b_s} x_{b_s}^{\top} - \rho_1 (1-\rho_1) x_{b_s} \zeta_s^{\top} - \rho_1 (1-\rho_1) \zeta_s x_{b_s}^{\top} + \rho_1^2 \zeta_s \zeta_s^{\top} \bigg) \nonumber\\& 
    \succeq \sum_{s \in N_t^c} \bigg(  - \rho_1 (1-\rho_1) x_{b_s} \zeta_s^{\top} - \rho_1 (1-\rho_1) \zeta_s x_{b_s}^{\top} + \rho_1^2 \zeta_s \zeta_s^{\top} \bigg) \nonumber\\& 
    = \sum_{s \in N_t^c} \bigg(  \rho_1^2 \mathbb{E} [\zeta_s \zeta_s^{\top}]    -  \rho_1 (1-\rho_1) x_{b_s} \zeta_s^{\top} -  \rho_1(1-\rho_1) \zeta_s x_{b_s}^{\top} + \rho_1^2 \zeta_s \zeta_s^{\top} - \rho_1^2 \mathbb{E} [\zeta_s \zeta_s^{\top}] \bigg) \nonumber\\&
    \succeq \rho_1^2 \sigma_{\zeta}^2 |N_t^c|  I + \sum_{s \in N_t^c} U_s,
\end{align} where $U_s$ is defined as \begin{align}
    U_s = \bigg(   -  \rho_1 (1-\rho_1) x_{b_s} \zeta_s^{\top} -  \rho_1 (1-\rho_1) \zeta_s x_{b_s}^{\top} + \rho_1^2 \zeta_s \zeta_s^{\top} - \rho_1^2 \mathbb{E} [\zeta_s \zeta_s^{\top}] \bigg). \label{defnition_of:U_s}
\end{align}
Then, using Weyl's inequality, it follows that \begin{align}
    \lambda_{\text{min}} (V_{t}) \geq \rho_1^2 \sigma_{\zeta}^2 |N_t^c| - \lambda_{\text{max}} (\sum_{s \in N_t^c} U_s). \nonumber
\end{align}

Next, we apply the matrix Azuma inequality (see Theorem \ref{matrix_azuma_inequality}) to find an upper bound on $\lambda_{\text{max}} (\sum_{s \in N_t^c} U_s)$. 
For this, we first need to show that the sequence of matrices $U_s$ satisfies the conditions of Theorem \ref{matrix_azuma_inequality}. By definition of $U_s$ in \eqref{defnition_of:U_s}, it follows that $\mathbb{E}[U_s | \mathcal{F}_{s-1}] = 0$, and $U_s^{\top} = U_s$. Also, we construct the sequence of deterministic matrices $A_s$ such that $U_s^2 \preceq A_s^2$ as follows. We know that for any matrix $B$, $B^2 \leq \|B\|_2^2 I$, where $\| B\|_2$ is the maximum singular value of $B$, i.e.,\begin{align}
    \sigma_{\text{max}}(B) = \max_{\|u\|_1 = \|v\|_2 = 1} u^{\top} B v. \nonumber
\end{align} Thus, we first show the following bound on the maximum singular value of the matrix $U_s$ defined in \eqref{defnition_of:U_s}: \begin{align}
   \max_{\|u\|_1 = \|v\|_2 = 1} u^{\top} U_s v  & =    -  \rho_1 (1-\rho_1) (u^{\top} x_{b_s}) (v^{\top} \zeta_s)^{\top}  -  \rho_1 (1-\rho_1) (u^\top \zeta_s) (v^\top x_{b_s})^{\top} + \nonumber\\& \quad \quad \rho_1^2 (u^\top \zeta_s) (v^\top \zeta_s)^{\top} - \rho_1^2 \mathbb{E} \left[(u^\top \zeta_s) (v^\top \zeta_s)^{\top}\right] \nonumber\\& \leq  \rho_1 (1-\rho_1) \| x_{b_s} \|_2 \| \zeta_s\|_2 + \rho_1 (1-\rho_1) \| \zeta_s\|_2 \| x_{b_s} \|_2  + \rho_1^2 \|\zeta_s\|_2^2 + \rho_1^2 \mathbb{E} \left[\|\zeta_s\|_2^2 \right] \nonumber \\& \leq 2 \rho_1 (1-\rho_1) L + 2 \rho_1^2, \label{upperbounding_maximum_singular_vqalue_of_Us}
\end{align} where we have used Cauchy-Schwarz inequality and the last inequality comes from the fact that $\| \zeta_s\|_2 = 1$ almost surely, and $\| x_{b_s}\|_2 \leq L$ by Assumption \ref{ass:actionset}. 
From the derivations above, and choosing $A_s = h_1 I$, with $h_1 =2 \rho_1 (1-\rho_1) L + 2 \rho_1^2 $, it almost surely holds that $U_s^2 \preceq \sigma_{\text{max}}(U_s)^2 I  \preceq h_1^2 I = A_s^2$. Moreover, using triangular inequality, it holds that \begin{align}
    \| \sum_{s \in N_t^c} A_s^2 \| \leq \sum_{s \in N_t^c} \| A_s^2 \| \leq |N_t^c| h_1^2. \nonumber
\end{align}

Now we apply the the matrix Azuma inequality, to conclude that for any $ c  \geq 0 $,  \begin{align}
    \mathbb{P} \left(  \lambda_{\text{max}} (\sum_{s \in N_t^c} U_s)  \geq c    \right) \leq d \exp{\left( - \frac{c^2}{8|N_t^c|h_1^2} \right)}. \nonumber
\end{align}
Therefore, it holds that with probability $1-\delta$, $\lambda_{\text{max}} (\sum_{s \in N_t^c} U_s) \leq \sqrt{8 |N_t^c| h_1^2 \log(\frac{d}{\delta})}$, and hence with probability $1-\delta$, \begin{align}
    \lambda_{\text{min}} (V_{t}) \geq \rho^2 |N_t^c| \sigma_{\zeta}^2  - \sqrt{8 |N_t^c| h_1^2 \log(\frac{d}{\delta})}, \nonumber
\end{align} or equivalently, 
\begin{align}
    \mathbb{P}(  \lambda_{\text{min}} (V_{t}) \leq t ) \leq d \exp{ \left( - \frac{(\rho_1^2 |N_t^c| \sigma_{\zeta}^2 - t)^2}{8 |N_t^c| h_1^2 } \right)}, \nonumber
\end{align} where $h_1 = 2 \rho_1 (1-\rho_1) L + 2 \rho_1^2 $ and $\rho_1 = (\frac{r_l}{S+r_h}) \alpha$. This completed the proof of lemma.

 \subsection{Matrix Azuma Inequality
}
\begin{theorem}[Matrix Azuma Inequality, 
\cite{tropp2012user}] \label{matrix_azuma_inequality}

Consider a  sequence $\{Y_k \}$ of independent, random matrices adapted to the filtration $\{ \mathcal{F}_k \}$. Each $\{Y_k \}$ is a self-adjoint matrix such that $\mathbb{E}[Y_k \,| \, \mathcal{F}_{k-1}] =0 $. Consider a fixed matrix $A_k$ such that $Y_k^2 \preceq A_k^2$ holds almost surely. Then, for $t \geq 0$, it holds that \begin{align}
    \mathbb{P} \left( \lambda_{\text{max}} \left(\sum_{k=1}^s Y_k\right) \geq t  \right) \leq d \exp{\left( - \frac{t^2}{8 \| \sum_{k=1}^s A_k^2 \|}  \right)}.
\end{align}
\end{theorem}

\subsection{Numerical analysis}
In order to numerically verify our results in Theorem \ref{upperbound:numberofbaseline}, we plot the cumulative number of time that baseline actions played bt SCLTS until time $t$ for $t = 1,\dots,1000$ over $100$ realizations.  The solid line in Figure \ref{fig:numberofbaseline} depicts average over $100$ realizations and the shaded regions show standard deviation. The figure confirms the logarithmic trend predicted by theory. 

\begin{figure}
     \centering
          \includegraphics[width=0.5\textwidth]{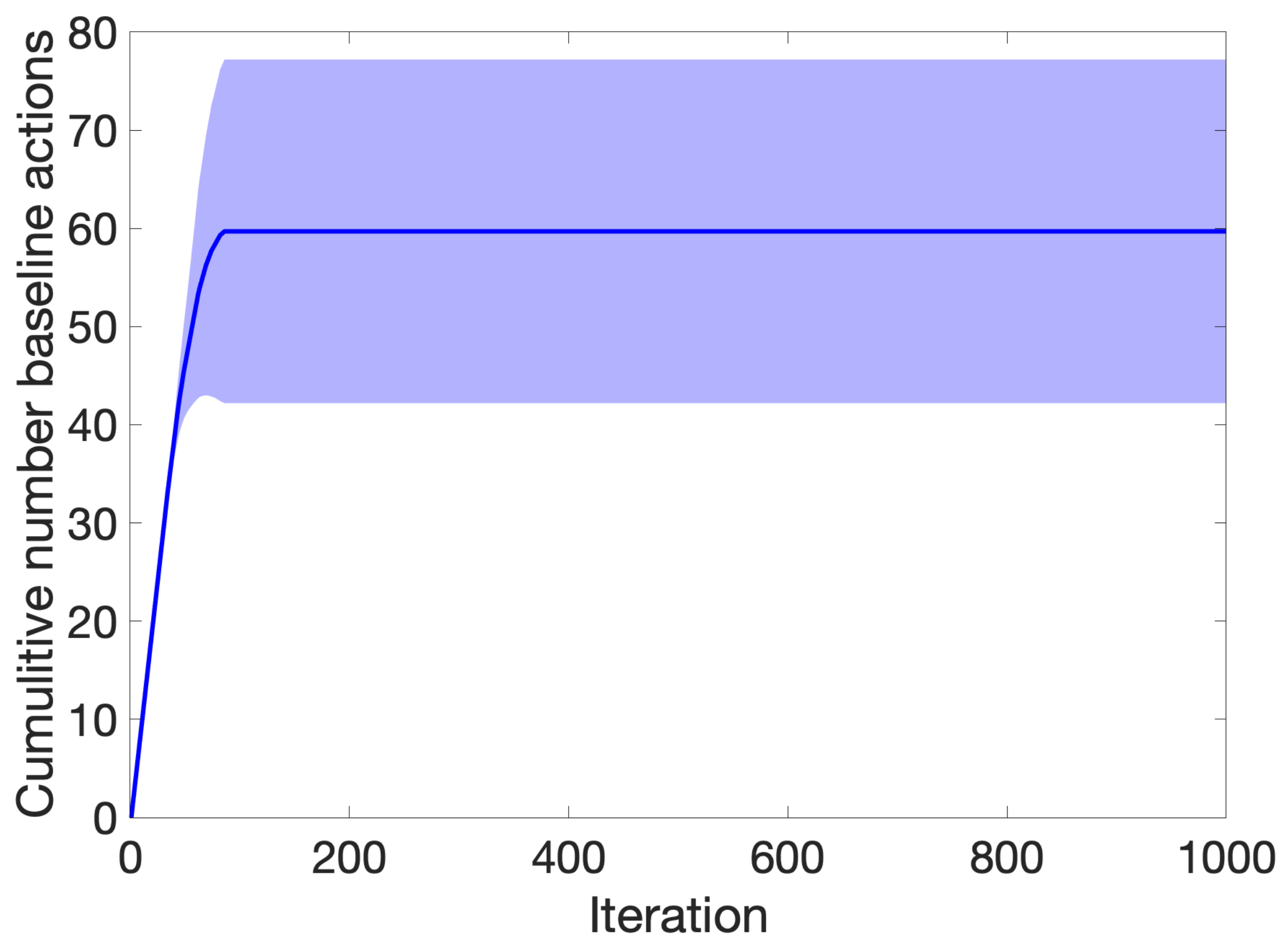}
         \caption{ Cumulative number of times that the baseline actions played by SCLTS up to time $t$, for $t=1\dots,1000$ over 100 realizations.  }\label{fig:numberofbaseline}
   \end{figure}

\section{Upper Bounding the Regret of SCLTS-BF}\label{SCLTS-BF:regret_proof}

In this section we provide the variation of our algorithm for the case of constraints with bandit feedback, which we refer to as SCLTS-BF in Algorithm \ref{alg:SCLTS-BF}. We then provide a regret bound for SCLTS-BF.  The summary of SCLTS-BF is presented in Algorithm \ref{alg:SCLTS-BF}.

 \begin{algorithm} 

\caption{SCLTS-BF}\label{alg:SCLTS-BF}

\textbf{Input:}  $\delta, T, \lambda, \rho$

Set $\delta' = \frac{\delta}{4T}$\\
\For{$t=1,\dots,T$ } { 
Sample $\eta_t \sim \mathcal{H}^{\text{TS}}$ \\

Compute RLS-estimate $\hat{\theta}_t$ and $V_t$ according to \eqref{RLS-estimate} and $\hat{\mu}_t$

Set $\tilde{\theta}_t = \hat{\theta}_t + \beta_t V_t^{-1/2} \eta_t$\\

 Build the confidence region $\mathcal{E}_t(\delta')$ in \eqref{app:sclts-bf:condifence_region_for_thetahat} and $\mathcal{C}_t(\delta')$ in \eqref{app:sclts-bf:conf_reg_on_mu_hat}
% \\

Compute the estimated safe set $\mathcal{P}_t^s = \{ x \in \mathcal{X} : \langle x, v \rangle \geq (1-\alpha) q_{b_t}, \forall v \in \mathcal{C}_t \} $  \\

 \textbf{if} the following optimization has a feasible solution: $x(\Tilde{\theta}_t) = \text{argmax}_{x \in \mathcal{P}_t^s} \langle x , \Tilde{\theta}_t \rangle$, \textbf{then}
 
 Set $F=1$, \textbf{else} $F = 0$ \\

\textbf{if} $F = 1$ \textbf{and} $\lambda_{\text{min}} (V_t)\geq \left( \frac{2 L \beta_t}{\nu_l + \alpha q_l}\right)^2$,  \textbf{then}

    Play $x_t = x(\Tilde{\theta}_t)$

\textbf{else}

    play $x_t =x_t^{\text{cb}}$ defined in \eqref{app:sclts-bf:conservative_ACVTIOPN}

 Observe reward $r_t$ \\ }\textbf{end for}
\SetAlgoLined
\end{algorithm}

In this setting, we assume that at each round $t$, with playing an action $x_t$, the learner observes the reward $ y_t = \langle x_t, \theta_{\star} \rangle + \xi_t$ and  the following bandit feedback: \begin{align}
    w_t = \langle x_t, \mu_\star \rangle + \chi_t, \label{app:sclts-bf:observ:bandi-feedback}
\end{align} where $\chi_t$ is assumed to be a zero-mean $R$-sub-Gaussian noise.

The main difference of SCLTS-BF with SCLTS is in the definition of the estimated safe action set. In particular, at each round $t$, SCLTS-BF constructs the following  confidence regions: \begin{align}
    & \mathcal{E}_t(\delta')  = \{ \theta \in \mathbb{R^d} : \norm{ \theta -\hat{\theta}_t}_{V_t} \leq \beta_t(\delta')   \}, \label{app:sclts-bf:condifence_region_for_thetahat}\\&
    \mathcal{C}_t(\delta')  = \{ v \in \mathbb{R^d} : \norm{ v-\hat{\mu}_t}_{V_t} \leq \beta_t(\delta')   \}, \label{app:sclts-bf:conf_reg_on_mu_hat}
\end{align} where $\hat{\mu}_t = V_t^{-1} \sum_{s=1}^{t-1} w_s x_s$ is the RLS-estimate of $\mu_\star$. The radius in \eqref{app:sclts-bf:condifence_region_for_thetahat} and \eqref{app:sclts-bf:conf_reg_on_mu_hat} is chosen according to Proposition \ref{abbasi-ellipsoid} such that $\theta_\star \in \mathcal{E}_t$ and $\mu_\star \in \mathcal{C}_t$ with high probability. In order to ensure safety at each round $t$, SCLTS-BF constructs the following estimated safe action set \begin{align}
    \mathcal{P}_t^s = \{ x \in \mathcal{X} : \langle x, v \rangle \geq (1-\alpha) q_{b_t}, \forall v \in \mathcal{C}_t \}.  \label{app:sclts-bfestimated_safe_set_for_bandit_feedback}
\end{align} The challenge with $\mathcal{P}_t^s$ is that it  contains all the actions that are safe with respect to all the parameters in $\mathcal{C}_t$. Thus, there may exist some rounds that $\mathcal{P}_t^s$ is empty. To handle this case, SCLTS-BF proceed as follows. At each round $t$, given the sampled parameter $\tilde{\theta}_t$, if the estimated safe action set $\mathcal{P}_t^s$ defined in \eqref{app:sclts-bfestimated_safe_set_for_bandit_feedback} is not empty, SCLTS-BF plays the safe action \begin{align}
    x(\tilde{\theta}_t) = \arg\max_{x \in \mathcal{P}_t^s} \langle x, \tilde{\theta}_t \rangle
\end{align} only if $\lambda_{\text{min}}(V_t) \geq k^2_t$, where $k^2_t = \left( \frac{2 L \beta_t}{\nu_l + \alpha q_l}\right)^2$. Otherwise, it plays the following conservative action  \begin{align}
    x_{t}^{\text{cb}} = (1-\rho_2) x_{b_t} + \rho_2 \zeta_t, \label{app:sclts-bf:conservative_ACVTIOPN}
\end{align} where $\rho_2 = \alpha (\frac{q_l}{S+q_h})$ in order to ensure that the conservative actions are safe.

Next,  we provide a regret guarantee for SCLTS-BF. First, we use the following decomposition of regret: \begin{align}
    R(T) & =  \sum_{t=1}^T \langle x_\star , \theta_\star \rangle - \langle x_t, \theta_\star \rangle  \nonumber \\&= \underbrace{ \sum_{t \in N_T} \bigg( \langle x_\star , \theta_\star \rangle - \langle x_t, \theta_\star \rangle \bigg)}_{\text{Term I}} + \underbrace{\sum_{t \in N_T^c} \bigg( \langle x_\star , \theta_\star \rangle - \langle (1-\rho) x_{b_t} - \rho \zeta_t, \theta_\star \rangle \bigg)}_{\text{Term II}},
\end{align}
 where $N_t^c$ is the set of rounds $i < t$ that SCLTS-BF plays the conservative actions, and $N_t = \{ 1,\dots,t\} - N_t^c$. In the following, we upper bound both Term I and Term II, separately. 

\textbf{Bounding Term I.} Bounding Term I follows the same steps as that of Theorem \ref{thm:bounding_term_I_regret_of_safe_lts}. Here, we  show that for SCLTS-BF, at rounds $t \in N_T$, the optimal action $x_\star$ belongs to the estimated safe safe, i.e., $x_\star \in \mathcal{P}_t^s$. Then, we conclude that regret of Term I similar to Theorem \ref{thm:bounding_term_I_regret_of_safe_lts} has the order of $\mathcal{O}(d^{3/2} \log^{1/2}d ~ T^{1/2} \log^{3/2}T)$.

At rounds $ t \in N_T$, we know \begin{align}
 \lambda_{\text{min}}(V_t) \geq k_t^2  \geq \left( \frac{2 L \beta_t}{\nu_{b_t} + \alpha q_{b_t}} \right)^2. \label{value_choosing_k2} 
\end{align}
Then, in order to show that $x_\star \in \mathcal{X}_t^s$, we need to show \begin{align}
    \langle x_\star , \hat{\mu}_t \rangle - \beta_t \| x_\star \|_{V_t^{-1}} \geq  \langle x_\star, \mu_\star \rangle - 2 \beta_t \| x_\star\|_{V_t^{-1}} \geq (1-\alpha) q_{b_t}.
\end{align}

First inequality comes from the fact that $\| \mu_\star - \hat{\mu}_t \|_{V_t} \leq \beta_t$. Therefore, it suffices to show the second inequality holds. We use the fact that $\| x_\star \|_{V_t^{-1}} \leq \frac{\|x_\star\|_2}{\sqrt{\lambda_{\text{min}}(V_t)}} \leq \frac{L}{\sqrt{\lambda_{\text{min}}(V_t)}}$, where  we use Assumption \ref{ass:actionset} to bound $\|x_\star\|_2$. Hence, we have \begin{align}
  \langle x_\star, \mu_\star \rangle - 2 \beta_t \| x_\star\|_{V_t^{-1}} \geq    \langle x_\star, \mu_\star \rangle - \frac{2 \beta_t L}{\sqrt{\lambda_{\text{min}}(V_t)}}.
\end{align}  Then, it suffices to show that \begin{align}
  \nu_{b_t}  + \alpha q_{b_t} \geq \frac{2 \beta_t L}{\sqrt{\lambda_{\text{min}}(V_t)}}, \label{showing_x*_in_pts}
\end{align}
 From \eqref{value_choosing_k2}, we know that \eqref{showing_x*_in_pts} holds, and hence, $x_\star \in \mathcal{P}_t^s$.  Therefore, we can use the result of Theorem \ref{thm:bounding_term_I_regret_of_safe_lts}, and obtain the desired regret bound. 

\textbf{Bounding Term II.}  First, we provide the formal statement of the theorem.

\begin{theorem}\label{thm:up_bound_SCLTS-bf_number_conservative}
Let $\lambda , L \geq 1$. On event $\bigg \{ \{ \theta_\star \in \mathcal{E}_t, \forall t \in [T] \} \cap \{ \mu_\star \in \mathcal{C}_t, \forall t \in [T] \} \bigg\} $, and  Assumptions \ref{ass:lowerbound_reward}, we can upper bound the number of times SCLTS-BF plays the conservative actions, i.e., $|N_T^c|$ as: \begin{align}
  |N_T^c|  \leq  \left( \frac{2L \beta_{T} }{\rho_2 \sigma_{\zeta}(\alpha q_l + \nu_l)} \right)^2 + \frac{2h_2^2}{\rho_2^4 \sigma_{\zeta}^4} \log(\frac{d}{\delta}) + \frac{2L h_2 \beta_T \sqrt{8 \log(\frac{d}{\delta})}}{\rho_2^3 \sigma_\zeta^3 (\alpha q_l + \nu_l)}
\end{align} where $h_2 = 2 \rho_2 (1-\rho_2) L + 2 \rho_2^2 $ and $\rho_2 = (\frac{q_l}{S+q_h}) \alpha$.
\end{theorem}

In order to prove Theorem \ref{thm:up_bound_SCLTS-bf_number_conservative}, we proceed as follows: 
\begin{align}
 \sum_{t \in N_T^c} \bigg( \langle x_\star , \theta_\star \rangle - \langle (1-\rho_2) x_{b_t} - \rho_2 \zeta_t, \theta_\star \rangle \bigg)  & =    \sum_{t \in N_T^c} \langle x_\star , \theta_\star \rangle - \langle x_{b_t}, \theta_\star \rangle + \rho_2(\langle x_{b_t} + \zeta_t , \theta_\star \rangle) \nonumber \\& \leq \sum_{t \in N_T^c} \nu_h + \rho_2 (q_{b_t} + S) \leq |N_T^c| (\nu_h+ \alpha q_l),
\end{align}where $q_h \geq q_{b_t}\geq q_l >0 $ and $\nu_h \geq \nu_{b_t} \geq \nu_l $ for all $t$.  Therefore, in order to bound Term II, it suffices to upper bound $|N_T^c|$ which is the number of rounds that SCLTS-BF plays the conservative actions up to round T. In order to do so, we proceed as follows: 

Let $\tau$ be any round that the algorithm plays the conservative action. 

If $F=0$, i.e., 
 \begin{align}
    \nexists x \in \mathcal{X} : \langle x , \hat{\mu}_\tau \rangle - \beta_\tau \| x \|_{V_\tau^{-1}}   \geq (1-\alpha) q_{b_\tau},
\end{align} and since we know that $x_\star \in \mathcal{X}$, and $\mu_\star \in \mathcal{C}_t$ with high probability, we can write  \begin{align}
 \langle x_\star , \mu_\star \rangle - 2\beta_\tau \| x_\star \|_{V_\tau^{-1}}  \leq \langle x_\star , \hat{\mu}_\tau \rangle - \beta_\tau \| x_\star \|_{V_\tau^{-1}}  < (1-\alpha) q_{b_\tau}. \label{bounding_constraint_value_IN_DIFFERENT_PARAMETER_for_SCLTS-bf}
\end{align} Using   \eqref{bounding_constraint_value_IN_DIFFERENT_PARAMETER_for_SCLTS-bf}, we can get \begin{align}
    \nu_{b_\tau} + \alpha q_{b_\tau} < 2 \beta_\tau \| x_\star \|_{V_\tau^{-1}} \leq \frac{2 \beta_\tau L}{\sqrt{\lambda_{\text{min}}(V_\tau)}},
\end{align} and hence the following upper bound on minimum eigenvalue of the Gram matrix: \begin{align}
    \lambda_{\text{min}}(V_\tau) < \left(\frac{2 \beta_\tau L}{\nu_{b_\tau} + \alpha q_{b_\tau}}\right)^2 \leq \left(\frac{2 \beta_\tau L}{\nu_l + \alpha q_l}\right)^2 =  k_\tau
\end{align}

Therefore, we show that in the cases where either the event $\{ \nexists x \in \mathcal{X} : \langle x , \hat{\mu}_\tau \rangle - \beta_\tau \| x \|_{V_\tau^{-1}}   \geq (1-\alpha) q_{b_\tau}\}$  or the event $\{ \lambda_{\text{min}}(V_\tau) < k_\tau^2 \}$ happen, we can conclude that at round $\tau$ \begin{align}
\lambda_{\text{min}}(V_\tau) < k_\tau^2. \label{getting_the_bound_for_number_in_unkwon_parameter}
\end{align}

From Lemma \ref{lemma:lowerbounding_lambdamin-vt}, we know that the minimum eigenvalue of the Gram matrix, i.e., $\lambda_{\text{min}}(V_t)$ is lower bounded with the number of times that SCLTS-BF plays the conservative actions, i.e., $|N_T^c|$. Therefore, using \eqref{getting_the_bound_for_number_in_unkwon_parameter}, we can get 
\begin{align}
  |N_T^c|  \leq  \left( \frac{2L \beta_{T} }{\rho_2 \sigma_{\zeta}(\alpha q_l + \nu_l)} \right)^2 + \frac{2h_2^2}{\rho_2^4 \sigma_{\zeta}^4} \log(\frac{d}{\delta}) + \frac{2L h_2 \beta_T \sqrt{2 \log(\frac{d}{\delta})}}{\rho_2^3 \sigma^3 (\alpha q_l + \nu_l)}
\end{align} where $h_2 = 2 \rho_2 (1-\rho_2) L + 2 \rho_2^2 $ and $\rho_2 = \alpha (\frac{q_l}{S+q_h})$.

\section{Proof of Theorem \ref{upperbound:unkownbaseline}}\label{app:sclts2:explantions}

In this section, we first present the SCLTS2 algorithm, for the case where the learner does not know the reward of the actions suggested by baseline policy in advance, i.e., $r_{b_t}$. The summary of SCLTS2 is presented in Algorithm \ref{alg:sclts2}.

The algorithm relies on the fact that we can find an upper bound over the value of $r_{b_t}$, using the fact that $\theta_\star \in \mathcal{E}_t$, i.e.,: \begin{align}
    \max_{ v \in \mathcal{E}_t} \langle x_{b_t}, v \rangle \geq \langle x_{b_t}, \theta_\star \rangle = r_{b_t}. \label{app:sclts2:upperboundin_the_unkown_reward}
\end{align} Then, we can write the safety constraint as follows:  \begin{align}
 \min_{v \in \mathcal{E}_t} \langle x(\Tilde{\theta}_t), v \rangle \geq (1 - \alpha) \max_{ v \in \mathcal{E}_t} \langle x_{b_t}, v \rangle. \label{app:sclts2:safety_consertvative_constrain_for unkown_reward}
\end{align} It is easy to show that safety constraint \eqref{cons:safety} holds when \eqref{app:sclts2:safety_consertvative_constrain_for unkown_reward} is true. Therefore, if we choose actions that satisfy \eqref{app:sclts2:safety_consertvative_constrain_for unkown_reward}, we can ensure that they are safe with respect to the safety constrain in \eqref{cons:safety}. 

Then we propose the estimated safe action set $\mathcal{Z}_t^s$ as: \begin{align}
    \mathcal{Z}_t^s & =  \{ x \in \mathcal{X} :\min_{v \in \mathcal{E}_t} \langle x, v \rangle \geq  (1 - \alpha) \max_{ v \in \mathcal{E}_t} \langle x_{b_t}, v \rangle\},
\end{align} which contains actions that are safe with respect to all the parameter in $\mathcal{E}_t$. At each round $t$, SCLTS2 plays the safe action $x(\tilde{\theta}_t)$ from  $\mathcal{Z}_t^s$ that maximizes the expected reward given the sampled parameter $\tilde{\theta}_t$, i.e., \begin{align}
    x(\Tilde{\theta}_t) = \arg\max_{x \in \mathcal{Z}_t^s} \langle x , \Tilde{\theta}_t \rangle
\end{align}
only if $\lambda_{\text{min}} (V_t) \geq k_t^3$, where $k_t^3 =\left( \frac{2L \beta_t (2-\alpha)}{\kappa_{l} + \alpha r_{l}} \right)^2 $. Otherwise it plays the conservative action $x_{b_t}^{\text{cb}}$ as:\begin{align}
   x_{t}^{\text{cb}} = (1- \rho_3)x_{b_t} + \rho_3 \zeta_t,\label{app:sclts2:conservative_actiosn}
\end{align} where $\rho_3=\alpha (\frac{r_l}{S+1}) $ such that the conservative action $x_{t}^{\text{cb}}$ is safe, where we use Assumption \ref{ass:actionset} for upper bounding the reward, i.e., $r_{b_t} \leq 1$.
\begin{algorithm} 

\caption{SCLTS2 }\label{alg:sclts2}

\textbf{Input:}  $\delta, T, \lambda, \rho$

Set $\delta' = \frac{\delta}{4T}$\\
\For{$t=1,\dots,T$ } { 
Sample $\eta_t \sim \mathcal{H}^{\text{TS}}$ \\

Compute RLS-estimate $\hat{\theta}_t$ and $V_t$ according to \eqref{RLS-estimate}

Set $\tilde{\theta}_t = \hat{\theta}_t + \beta_t V_t^{-1/2} \eta_t$\\

 Build the confidence region $\mathcal{E}_t(\delta')$ in \eqref{confidence_region_unkown_reward_parameter} 
% \\

Compute the estimated safe set $ \mathcal{Z}_t^s  =  \{ x \in \mathcal{X} :\min_{v \in \mathcal{E}_t} \langle x, v \rangle \geq  (1 - \alpha) \max_{ v \in \mathcal{E}_t} \langle x_{b_t}, v \rangle\}$ \\

 \textbf{if} the following optimization is feasible: $x(\Tilde{\theta}_t) = \arg\max_{x \in \mathcal{Z}_t^s} \langle x , \Tilde{\theta}_t \rangle$, \textbf{then}
 
 Set $F=1$, \textbf{else} $F = 0$ \\

\textbf{if} $F = 1$ \textbf{and} $\lambda_{\text{min}} (V_t)\geq \left( \frac{2L \beta_t (2-\alpha)}{\kappa_{l} + \alpha r_{l}} \right)^2$,  \textbf{then}

    Play $x_t = x(\Tilde{\theta}_t)$

\textbf{else}

    play $x_t =x_t^{\text{cb}}$ defined in \eqref{app:sclts2:conservative_actiosn}

 Observe reward $y_t$ \\ }\textbf{end for}
\SetAlgoLined
\end{algorithm}

In order to bound the regret of SCLTS2, we first use the decomposition defined  in Proposition \ref{decompostion_regret}. The regret of Term I is similar to that of SCLTS (i.e., Theorem \ref{thm:bounding_term_I_regret_of_safe_lts}). Hence, it suffices to  upper bound  the number of time SCLTS2 plays the conservative actions, i.e., $|N_T^c|$. 

In order to bound $|N_T^c|$, we proceed as follows. 
Let $\tau$ be the round that SCLTS2 plays a conservative action. If $F=0$, i.e., \begin{align}
    \nexists x \in \mathcal{X} : \min_{v \in \mathcal{C}_{\tau}} \langle x, v \rangle \geq (1 - \alpha) \max_{ v \in \mathcal{C}_{\tau}} \langle x_{b_{\tau}}, v \rangle.   \label{safetyforunkownreward}
\end{align}

Using the fact that $x_\star \in \mathcal{X}$, we can write 
\begin{align}
    \langle x_\star, \hat{\theta}_{\tau} \rangle - \beta_{\tau} \| x_\star \|_{V_{\tau}^{-1}} < (1-\alpha) \left( \langle x_{b_\tau}, \hat{\theta}_{\tau} \rangle   + \beta_{\tau} \| x_{b_{\tau}} \|_{V_{\tau}^{-1}}  \right). \label{unkownsafety:elliposid_equality}
\end{align}Then, since $\|\theta_{\star} - \hat{\theta}_t \|_{V_t} \leq \beta_t$,  we can upper bound the RHS and lower bound the LHS of \eqref{unkownsafety:elliposid_equality}, and get
\begin{align}
     \langle x_\star, {\theta}_{\star} \rangle - 2\beta_{\tau} \| x_\star \|_{V_{\tau}^{-1}} < (1-\alpha) \left( \langle x_{b_\tau}, {\theta}_{\star} \rangle   + 2 \beta_{\tau} \| x_{b_{\tau}} \|_{V_{\tau}^{-1}}  \right), \label{unkownsafety:shrunk_version}
\end{align} or equivalently,
\begin{align}
    \kappa_{b_\tau} + \alpha r_{b_\tau} <  2\beta_{\tau} \| x_\star \|_{V_{\tau}^{-1}} +  2 (1-\alpha) \beta_{\tau} \| x_{b_{\tau}} \|_{V_{\tau}^{-1}}.  \label{finding_an_upper_bound_for_unkown_reward_case}
\end{align} Then we can use the fact that $ \| x_\star \|_{V_{\tau}^{-1}} \leq \frac{L}{\sqrt{\lambda_{\text{min}}(V_\tau)}}$ and $ \| x_{b_{\tau}} \|_{V_{\tau}^{-1}} \leq \frac{L}{\sqrt{\lambda_{\text{min}}(V_\tau)}}$, where we use Assumption \ref{ass:actionset} for upper bounding $\|x_\star\|_2$. Thus, we upper bound the RHS of \eqref{finding_an_upper_bound_for_unkown_reward_case} as follows: \begin{align}
     \kappa_{b_\tau} + \alpha r_{b_\tau} < 2\beta_{\tau} \frac{L}{\sqrt{\lambda_{\text{min}}(V_\tau)}} +  2 (1-\alpha) \beta_{\tau}\frac{L}{\sqrt{\lambda_{\text{min}}(V_\tau)}},
\end{align}
and hence, we can get the following upper bound ${\lambda_{\text{min}}(V_\tau)}$ as follows: \begin{align}
  {\lambda_{\text{min}}(V_\tau)} < \left( \frac{2L \beta_T (2-\alpha)}{\kappa_{b_\tau} + \alpha r_{b_\tau}} \right)^2 \leq \left( \frac{2L \beta_T (2-\alpha)}{\kappa_{l} + \alpha r_{l}} \right)^2 = k_\tau^3 .  \label{the_upper_bound_for_unkown_case}
\end{align}

Therefore, we show that whether the event $F=0$ happens or ${\lambda_{\text{min}}(V_t)} < k_t^3$, we can achieve the upper bound provided in \eqref{the_upper_bound_for_unkown_case}.  
Then, using the result of Lemma \ref{lemma:lowerbounding_lambdamin-vt}, where we show that ${\lambda_{\text{min}}(V_t)}$ is lower bounded with the number of times the algorithm plays the conservative actions,  we  obtain the following upper bound on the $|N_{\tau}^c|$ 
\begin{align}
| N_{\tau}^c| \leq  \left( \frac{2 L  \beta_{\tau} (2 - \alpha)}{\rho_3 \sigma_{\zeta}(\kappa_l + \alpha r_l  )} \right)^2   + \frac{2h_3^2}{\rho_3^4 \sigma_{ \zeta}^4} \log(\frac{d}{\delta}) + 
\frac{ 2L h_3 \beta_\tau (2-\alpha)  }{\rho_3^3 \sigma_{\zeta}^3 (\kappa_l+\alpha r_l )} \sqrt{2 \log(\frac{d}{\delta})},
\end{align} where $h_3 = 2 \rho_3 (1-\rho_3) L + 2 \rho_3^2$ and $\rho_3 = \alpha (\frac{r_l}{S+1})$.

\section{Stage-wise Conservative Linear UCB (SCLUCB) Algorithm}\label{SCLUCB}
In this section we propose a UCB-based safe stochastic linear bandit algorithm called Stage-wise Conservative Linear-UCB (SCLUCB), which is a safe counterpart of LUCB for the  stage-wise conservative bandit setting. In particular,  at each round $t$, given the RLS-estimate $\hat{\theta}_t$ of $\theta_\star$, SCLUCB constructs the confidence region $\mathcal{E}_t$ as follows: \begin{align}
\mathcal{E}_t(\delta) = \{ \theta \in \mathbb{R}^d: \| \theta - \hat{\theta}_t \|_{V_t} \leq \beta_t(\delta) \}. \label{sclucb:confidence_region}
\end{align} The radius $\beta_t(\delta)$ is chosen as in Proposition \ref{abbasi-ellipsoid} such that $\theta_\star \in \mathcal{E}_t(\delta)$ with probability $1 - \delta$. Then, similar to SCLTS, it builds the estimated safe set $\mathcal{X}_t^s$ such that it includes actions that are safe with respect to all the parameter in $\mathcal{E}_t$, i.e., \begin{align}
    \mathcal{X}_t^s = \{x \in \mathcal{X} : \langle x, v \rangle \geq (1-\alpha) r_{b_t}, \forall v \in \mathcal{E}_t \}. \label{sclucb:estimated_safe_actoion_set}
\end{align}
Similar to SCLTS, the challenge with $\mathcal{X}_t^s$ is that there may exist some rounds that $\mathcal{X}_t^s$ is empty. In order to face this problem, SCLUCB proceed as follows.   In order to guarantee safety, at each round $t$, if $\mathcal{X}_t^s$ is not empty, SCLUCB plays the action $\bar{x}_t$ as \begin{align}
   ( \bar{x}_t, \bar{\theta}_t )= \max_{x \in \mathcal{X}_t^s} \max_{v \in \mathcal{E}_t} \langle x , v \rangle \label{sclucb:action_selection_rule}
\end{align}
only if $\lambda_{\text{min}}(V_t) \geq \left(\frac{2 L \beta_t}{\kappa_l + \alpha r_{b_l}}\right) ^2$, otherwise it plays the conservative action  $ x_t^{\text{cb}}$ defined in \eqref{conservative-actions}. The summary of SCLUCB is presented in Algorithm \eqref{alg:SCLUCB}.

\begin{algorithm} 

\caption{Stage-wise Conservative Linear UCB (SCLUCB) }\label{alg:SCLUCB}

\textbf{Input:}  $\delta, T, \lambda, \rho$

\For{$t=1,\dots,T$ } {

Compute RLS-estimate $\hat{\theta}_t$ and $V_t$ according to \eqref{RLS-estimate}

 Build the confidence region $\mathcal{E}_t(\delta)$ in \eqref{sclucb:confidence_region} 
% \\

Compute the estimated safe set $\mathcal{X}_t^s$ in \eqref{sclucb:estimated_safe_actoion_set} \\

 \textbf{if} the following optimization is feasible: $\bar{x}_t = \arg\max_{x \in \mathcal{X}_t^s} \max_{v \in \mathcal{E}_t} \langle x , v \rangle$, \textbf{then}
 
 Set $F=1$, \textbf{else} $F = 0$ \\

\textbf{if} $F = 1$ \textbf{and} $\lambda_{\text{min}} (V_t)\geq \left(\frac{2 L \beta_t}{\kappa_l + \alpha r_{b_l}}\right) ^2$,  \textbf{then}

    Play $x_t =\bar{x}_t$

\textbf{else}

    play $x_t =x_t^{\text{cb}}$ defined in \eqref{conservative-actions}

 Observe reward $y_t$ \\ }\textbf{end for}
\SetAlgoLined
\end{algorithm}

Next, we provide the regret guarantee for SCLUCB. Recall, $N_{t-1}$ be the set of rounds $i < t$ at which SCLUCB plays the action in \eqref{optimal-safe-action}. Similarly, $N_{t-1}^{c} = \{1,\dots,t-1\} - N_{t-1}$ is the set of rounds $j < t$ at which SCLUCB plays the conservative actions. 

\begin{Proposition}\label{SCLUCB:decompostion_regret}
The regret of SCLUCB can be decomposed into two terms as follows: \begin{align}
    R(T) \leq \underbrace{\sum_{t \in N_T} \left(\langle x_{\star}, \theta_{\star} \rangle - \langle x_t, \theta_{\star} \rangle \right)}_{\text{Term I}} +  \underbrace{|N_{T}^{{c}}| \left(\kappa_h + \rho_1 (r_h + S) \right)}_{\text{Term II}}\label{SCLUCB:regret:decompos}
\end{align}
\end{Proposition}

In the following, we bound both terms, separately.

\textbf{Bounding Term I.} The first Term in \eqref{SCLUCB:regret:decompos} is the regret caused by playing the safe actions that maximize the reward given the true parameter is $\bar{\theta}_t$. The idea of bounding Term I is similar to \cite{abbasi2011improved}. We use the fact that for $t \in N_T$, $x_t = \bar{x}_t$, and start with the following decomposition of the instantaneous regret for $t \in N_T$ : \begin{align}
 \langle x_{\star}, \theta_{\star} \rangle - \langle x_t, \theta_{\star} \rangle = \underbrace{ \langle x_{\star}, \theta_{\star} \rangle - \langle \bar{x}_t, \bar{\theta}_t \rangle}_{\text{Term A}} + \underbrace{ \langle \bar{x}_t, \bar{\theta}_t \rangle - \langle \bar{x}_t, \theta_{\star} \rangle}_{\text{Term B}}
\end{align} 
\textbf{Bounding Term A.} Since for round $t \in N_t$, we require that $\lambda_{\text{min}}(V_t) \geq k_t^1$, where $k_t^1 = \left(\frac{2 L \beta_t}{\kappa_l + \alpha r_{b_l}}\right) ^2 $, we can conclude that $x_\star \in \mathcal{X}_t^s$. Therefore, due to \eqref{sclucb:action_selection_rule}, we have $\langle \bar{x}_t, \bar{\theta}_t \rangle \geq \langle x_{\star}, \theta_{\star} \rangle$, and hence Term A is not positive. 

\textbf{Bounding Term B.} In order to bound Term B, we use the following chain of inequalities: \begin{align}
    \text{Term B} & :=  \langle \bar{x}_t, \bar{\theta}_t \rangle - \langle \bar{x}_t, \theta_{\star} \rangle = \langle \bar{x}_t, \bar{\theta}_t \rangle -\langle \bar{x}_t, \hat{\theta}_t  \rangle + \langle \bar{x}_t, \hat{\theta}_t  \rangle - \langle \bar{x}_t, \theta_{\star} \rangle \nonumber \\& \leq \|\bar{x}_t \|_{V_t^{-1}} \| \bar{\theta}_t - \hat{\theta}_t \|_{V_t} + \| \bar{x}_t\|_{V_t^{-1}} \| \hat{\theta}_t - \theta_\star \|_{V_t} \nonumber\\& \leq 2 \beta_t \| \bar{x}_t\|_{V_t^{-1}},
\end{align} where the last inequality follows from Proposition \ref{abbasi-ellipsoid}. Recall, from Assumption \ref{ass:actionset}, we have the following trivial bound: \begin{align}
  \langle x_\star, \theta_{\star} \rangle - \langle \bar{x}_t, \theta_{\star} \rangle  \leq 2.  
\end{align}
Thus, we conclude the following \begin{align}
\text{Term B} \leq 2 \min(\beta_t \| \bar{x}_t\|_{V_t^{-1}}, 1).
\end{align}
Next, we state a direct application of Lemma 11 in \cite{abbasi2011improved}. 

\begin{lemma}\label{ucb_bound_of_abbasi}
For $\lambda > 0$, and under Assumptions \ref{ass:sub-gaussian_noise}, \ref{ass:bounded_parameter}, and \ref{ass:actionset}, we have \begin{align}
    \sum_{t=1}^T \min( \| \bar{x}_t\|^2_{V_t^{-1}}, 1) \leq 2d \log\left(  1 + \frac{T L^2}{\lambda d}\right)
\end{align}
\end{lemma}

Therefore, from Lemma \ref{ucb_bound_of_abbasi}, we can conclude the following bound on regret of Term B: \begin{align}
\sum_{t \in N_T} 2 \min(\beta_t \| \bar{x}_t\|_{V_t^{-1}}, 1) \leq 2 \beta_T \sqrt{ 2 d |N_T| \log(1 + \frac{|N_T| L^2}{\lambda d} )}.
\end{align}

Next, in Theorem \ref{SCLUCB:thm:boudning_regret_term_one}, we provide an upper bound on the regret of Term I which is of order $\mathcal{O}\left( d \sqrt{T} \log(\frac{T L^2}{\lambda \delta}) \right)$. 

\begin{theorem}\label{SCLUCB:thm:boudning_regret_term_one}
On event $\{\theta_\star \in \mathcal{E}_t\}$ for a fixed $\delta \in (0,1)$, with probability $1-\delta$, it holds that: \begin{align}
    \sum_{t \in N_T} \left(\langle x_{\star}, \theta_{\star} \rangle - \langle x_t, \theta_{\star} \rangle \right) \leq 2 \beta_T \sqrt{ 2 d T \log(1 + \frac{T L^2}{\lambda d} )}
\end{align} 
\end{theorem}

\textbf{Bounding Term II.} In order to bound Term II in \eqref{SCLUCB:regret:decompos}, we need to find an upper bound on the number of times that SCLUCB plays the conservative actions up to time $T$, i.e., $|N_T^c|$. We prove an upper bound on $|N_T^c|$ in Theorem \ref{sclucb:upperbound:number_of_baseline} which has the order of $\mathcal{O} \left(  \frac{L^2 d \log(\frac{T}{\delta})  \log(\frac{d}{\delta})    }{\alpha^4  (r_l^2 \wedge r_l^4)\kappa_l ( \sigma_{\zeta}^2 \wedge \sigma_{\zeta}^4) } \right)$. 

\begin{theorem}\label{sclucb:upperbound:number_of_baseline}
Let $\lambda , L \geq 1$. On event $\{ \theta_\star \in \mathcal{E}_t, \forall t \in [T] \}$, and under Assumption \ref{ass:lowerbound_reward},  we can upper bound the number of times SCLUCB plays the conservative actions, i.e., $|N_T^c|$ as: \begin{align}
         | N_{T}^c| \leq  \left( \frac{2L  \beta_{T} }{\rho_1 \sigma_{\zeta}( \kappa_l + \alpha r_l )} \right)^2 + \frac{2h_1^2}{\rho_1^4 \sigma_{\zeta}^4} \log(\frac{d}{\delta}) + \frac{2 Lh_1 \beta_T  \sqrt{8 \log(\frac{d}{\delta})}}{\rho_1^3 \sigma^3 (\kappa_l + \alpha r_l )},
\end{align} where $h_1 = 2 \rho_1 (1-\rho_1) L + 2 \rho_1^2 $ and $\rho_1 = (\frac{r_l}{S+r_h}) \alpha$.
\end{theorem}
The proof is similar to that of Theorem \ref{upperbound:numberofbaseline}, and we omit its proof here.

\section{Comparison with Safe-LUCB}\label{SCLUCB2:COMPARISON_SECTION}
In this section, we extend our results to   an alternative safe bandit formulation proposed in \cite{amani2019linear}, where the algorithm   Safe-LUCB  was proposed. 
In order to do so, we first present the safety constraint in \cite{amani2019linear}, and then we show the required modification of SCLUCB to handle this case, which we refer to as SCLUCB2. Then, we provide a problem-dependent regret bound for SCLUSB2, and we show that it matches the problem dependent regret bound of Safe-LUCB in \cite{amani2019linear}. We need to note that in \cite{amani2019linear}, they also provide a general regret bound of order $\tilde{\mathcal{O}} ({T^{2/3}})$ for Safe-LUCB which we do not discuss in this paper. 

In \cite{amani2019linear}, it is assumed that  the learner is given a convex and compact decision set $\mathcal{D}_0$ which contains the origin, and with playing the action $x_t$, she observes the reward of $y_t = x_t^{\top} \theta_\star + \eta_t$, where $\theta_\star$ is the fixed unknown parameter, and $\eta_t$ is $R$-sub-Gaussian noise. Moreover, The learning environment is subject to the linear safety constraint \begin{align}
    x^{\top} B \theta_\star \leq C, \label{safe-lucb:safety_constraint}
\end{align} which needs to be satisfied at all rounds $t$ with high probability, and an action $x_t$ is called safe, if it satisfies \eqref{safe-lucb:safety_constraint}. In \eqref{safe-lucb:safety_constraint}, the matrix $B \in \mathbb{R}^{d\times d}$ and the positive constant $C$ are known to the learner. However, the learner does not receive any bandit feedback on   the value $ x^{\top} B \theta_\star$ and her information is  restricted to those she receives  from  the reward.

Given the above constraint, the learner is restricted to choose actions from the safe set $\mathcal{D}_0^s$ as: \begin{align}
   \mathcal{D}_0^s (\theta_\star) = \{ x \in \mathcal{D}_0:  x^{\top} B \theta_\star \leq C  \}. 
\end{align} Since $\theta_\star$ in unknown, the safe set $\mathcal{D}_0^s$ is unknown to the learner. Then, in \cite{amani2019linear}, they provide the problem-dependent regret bound (for the case where $\Delta := C - x^{\top} B \theta_\star > 0 $) of order $\mathcal{O}(\sqrt{T}\log T)$. In the following, we present the required modification of SCLUSB to handle this safe bandit formulation, and propose the new algorithm called SCLUCB2 that we prove a problem dependnt regret bound of order $\mathbb{O}(\sqrt{T} \log T)$. We need to note that \cite{amani2019linear} also provide a general regret bound of order $\tilde{\mathcal{O}} ({T^{2/3}})$ for the case where  $\Delta = 0$; however, we do not discuss this case in this paper.

At each round $t$, given the RLS-estimate $\hat{\theta}_t$ of $\theta_\star$, SLUCB2 builds the confidence region $\mathcal{E}_t$ as: \begin{align}
    \mathcal{E}_t = \{\theta \in \mathbb{R}^d: \| \theta - \hat{\theta}_t \|_{V_t} \leq \beta_t  \}, \label{SCLUCB2:confidence_region}
\end{align} and the radius $\beta_t$ is chosen according to Proposition \ref{abbasi-ellipsoid} such that $\theta_\star \in \mathcal{E}_t$ with high probability. The learner does not know the safe set $\mathcal{D}_0^s$; however, she knows that $\theta_\star \in \mathcal{E}_t$ with high probability. Hence, SLUCB2  constructs the estimated safe  set $\mathcal{D}_t^s$ such that it contains actions that are safe with respect to all the parameter in $\mathcal{E}_t$, i.e., \begin{align}
    \mathcal{D}_t^s &= \{x \in \mathcal{D}_0 :  x^{\top} B v \leq C, \forall v \in \mathcal{E}_t \} \nonumber \\& = \{x \in \mathcal{D}_0 :  \max_{v \in \mathcal{E}_t }x^{\top} B v \leq C \} \nonumber \\& = \{x \in \mathcal{D}_0 :  x^{\top} B \hat{\theta}_t + \beta_t \|B x\|_{V_t^{-1}} \leq C \}
    \label{SCLUCB2:safe_estimated_actions}
\end{align}

Clearly, action $x = [0]^d$ (origin) is a safe action since  $C>0$, and also $[0]^d \in \mathcal{D}_0$. Thus, $ [0]^d \in \mathcal{D}_t^s$. 
Since $x = [0]^d$ is a known safe action, we define the conservative action $x_0^{\text{c}}$ as: \begin{align}
    x_0^{\text{c}} = (1-\rho) [0]^d + \rho \zeta_t = \rho \zeta_t, 
\end{align}where $\zeta_t$ is a sequence of IID random vectors such that $\| \zeta_t \|_2 = 1$ almost surly, and $\sigma_\zeta = \lambda_{\text{min}}(\text{Cov}(\zeta_t)) > 0$. We choose the constant $\rho$ according to the Lemma \ref{sclucb:upperbounding_rho} in order to ensure that the conservative action $x_0^{\text{c}}$ is safe. 

\begin{lemma}\label{sclucb:upperbounding_rho}
At each round $t$, for any $\rho \in (0,\bar{\rho})$, where \begin{align}
    \bar{\rho} = \frac{C}{\|B\| S},
\end{align} the conservative action $x_0^{\text{c}} = \rho \zeta_t$ is guaranteed to be safe almost surly.  
\end{lemma}
We choose $\rho = \frac{C}{\|B\| S}$ for the rest of this section, and hence the conservative action is \begin{align}
    x_0^{\text{c}} = \frac{C}{\|B\| S} \zeta_t. \label{sclucb2:defn:conservative_action}
\end{align}

Let $\Delta = C - x_\star^{\top} B \theta_\star$. We consider the case where $\Delta > 0$.
At each $t$, in order to guarantee safety, SCLUCB2 only chooses its action from the estimated safe set $\mathcal{D}_t^s$. The challenge with $\mathcal{D}_t^s$ is that it includes actions that are safe with respect to all parameter in $\mathcal{E}_t$, and not only $\theta_\star$. Thus, there may exist some rounds that $\mathcal{D}_t^s$ is empty. At round $t$, if $\mathcal{D}_t^s$ is not empty, SCLUCB2  plays the  safe action \begin{align}
    \bar{x}_t = \arg\max_{x \in \mathcal{D}_t^s} \max_{v \in \mathcal{E}_t} \langle x , v \rangle \label{sclucb2":optimal_safeaction}
\end{align} only if $\lambda_{\text{min}} (V_t)\geq \left(\frac{2 L \beta_t \| B\|}{\Delta}\right) ^2$, otherwise it plays the conservative action $x_0^{\text{c}}$ in \eqref{sclucb2:defn:conservative_action}. The summary of SCLUCB2 is presented in Algorithm \ref{SCLUCB2}.

\begin{algorithm} 

\caption{SCLUCB2 }\label{SCLUCB2}

\textbf{Input:}  $\delta, T, \lambda, \rho$

\For{$t=1,\dots,T$ } {

Compute RLS-estimate $\hat{\theta}_t$ and $V_t$ according to \eqref{RLS-estimate}

 Build the confidence region $\mathcal{E}_t(\delta)$ in \eqref{SCLUCB2:confidence_region} 
% \\

Compute the estimated safe set $\mathcal{D}_t^s$ in \eqref{SCLUCB2:safe_estimated_actions} \\

 \textbf{if} the following optimization is feasible: $\bar{x}_t = \arg\max_{x \in \mathcal{D}_t^s} \max_{v \in \mathcal{E}_t} \langle x , v \rangle$, \textbf{then}
 
 Set $F=1$, \textbf{else} $F = 0$ \\

\textbf{if} $F = 1$ \textbf{and} $\lambda_{\text{min}} (V_t)\geq \left(\frac{2 L \beta_t \| B\|}{\Delta}\right) ^2$,  \textbf{then}

    Play $x_t =\bar{x}_t$

\textbf{else}

    play $x_t =x_0^{\text{c}}$ defined in \eqref{sclucb2:defn:conservative_action}

 Observe reward $y_t$ \\ }\textbf{end for}
\SetAlgoLined
\end{algorithm}

In the following we provide the regret guarantee for SCLUCB2. Let $N_{t-1}$ be the set of rounds $i < t$ at which SCLUCB2 plays the action in \eqref{sclucb2":optimal_safeaction}. Similarly, $N_{t-1}^{c} = \{1,\dots,t-1\} - N_{t-1}$ is the set of rounds $j < t$ at which SCLUCB2 plays the conservative action in \eqref{sclucb2:defn:conservative_action}. 

First, we use the following decomposition of the regret, then we bound each term separately.
\begin{Proposition}
The regret of SCLUCB2 can be decomposed to the following two terms:
\begin{align}
    R(T) & =  \sum_{t=1}^T \langle x_\star , \theta_\star \rangle - \langle x_t, \theta_\star \rangle  \nonumber \\&= { \sum_{t \in N_T} \bigg( \langle x_\star , \theta_\star \rangle - \langle x_t, \theta_\star \rangle \bigg)} + {\sum_{t \in N_T^c} \bigg( \langle x_\star , \theta_\star \rangle - \langle  x_0^{\text{c}}, \theta_\star \rangle \bigg)}, \nonumber \\& \leq \underbrace{ \sum_{t \in N_T} \bigg( \langle x_\star , \theta_\star \rangle - \langle x_t, \theta_\star \rangle \bigg)}_{\text{Term I}} + \underbrace{ 2 |N_t^c|}_{\text{Term II}}. \label{sclucb2:decomopregert}
\end{align}
\end{Proposition}

\textbf{Bounding Term I.} In order to bound Term I, we proceed as follows. First, we show that at rounds $t \in N_T$, the optimal action $x_\star $ belongs to the estimated safe set $\mathcal{D}_t^s$, i.e., $x_\star \in \mathcal{D}_t^s$. To do so, we need to show that \begin{align}
   x_\star^{\top} B \hat{\theta}_t + \beta_t \|B x_\star \|_{V_t^{-1}} \leq C.\label{sclusb2:showing_x8_is_safe}
\end{align} Since $\| \theta_\star-\hat{\theta}_t\|_{V_t} \leq \beta_t$, it suffices to show that: \begin{align}
    x_\star^{\top} B \theta_\star + 2\beta_t \|B x_\star \|_{V_t^{-1}} \leq C, \label{sclusb2:suffces_to_show_for_x*_safe_set}  
\end{align} or equivalently \begin{align}
    2\beta_t \|Bx_\star \|_{V_t^{-1}} \leq \Delta, \label{SLCUSB2:toboundlambadmin}
\end{align} where $\Delta = C - x_\star^{\top} B \theta_\star $.
It is easy to see \eqref{sclusb2:showing_x8_is_safe} is true whenever \eqref{sclusb2:suffces_to_show_for_x*_safe_set} holds. Using Assumption \ref{ass:actionset}, we can get $\|Bx_\star\|_{V_t^{-1}} \leq \frac{\|B\| \|x_\star\|_2}{\sqrt{\lambda_{\text{min}}(V_t)}} \leq \frac{\|B\| L}{\sqrt{\lambda_{\text{min}}(V_t)}}$. Hence, from \eqref{SLCUSB2:toboundlambadmin}, it suffices to show that  \begin{align}
   \frac{ 2 \beta_t \|B\| L}{\sqrt{\lambda_{\text{min}}(V_t)}} \leq \Delta, 
\end{align}or equivalently \begin{align}
    \lambda_{\text{min}}(V_t) \geq \left(\frac{2 \beta_t \|B\| L}{\Delta}\right)^2
\end{align} that we know it is true for $t \in N_T$. Therefore, on event $\{\theta_\star \in \mathcal{E}_t \}$, $x_\star \in \mathcal{D}_t^s$. We can bound the regret of Term I in \eqref{sclucb2:decomopregert} similar to Theorem \ref{SCLUCB:thm:boudning_regret_term_one}, and get the regret of order $\mathcal{O}\left( d \sqrt{T} \log(\frac{T L^2}{\lambda \delta}) \right)$. 

\textbf{Bounding Term II.} We need to upper bound the number of times that SCLUCB2 plays the conservative action $x_0^{\text{c}}$, i.e., $|N_T^c|$. We prove an upper bound on $|N_T^c|$ in Theorem \ref{SCLUCB2:THM:boundon_number} which has the order of $\mathcal{O} \left(  \frac{L^2 S^2 \|B\|^2 d \log(\frac{T}{\delta})  \log(\frac{d}{\delta})    }{\Delta^2 ( C \wedge C^2) ( \sigma_{\zeta}^2 \wedge \sigma_{\zeta}^4) } \right)$. 

\begin{theorem}\label{SCLUCB2:THM:boundon_number}
Let $\lambda , L \geq 1$. On event $\{ \theta_\star \in \mathcal{E}_t, \forall t \in [T] \}$,   we can upper bound the number of times SCLUCB2 plays the conservative actions, i.e., $|N_T^c|$ as: \begin{align}
         | N_{T}^c| \leq  \left(\frac{2 L S \| B\|^2\beta_T }{C \Delta \sigma_\zeta }\right)^2 + \frac{32 \log(\frac{d}{\delta})}{\sigma_\zeta^4} + \frac{8 L S \| B\|^2 \beta_T \sqrt{2 \log(\frac{d}{\delta})}}{C \Delta \sigma_\zeta^3}.
\end{align}
\end{theorem}
\begin{proof}
Let $\tau$ be any round that the algorithm plays the conservative action, i.e., at round $\tau$, either $F = 0$ or $\lambda_{\text{min}}(V_\tau) < \left(\frac{2L \|B\| \beta_\tau}{\Delta} \right)^2$.

By definition, if $F=0$, we have 
\begin{align}
    \nexists x \in \mathcal{X} :  x^{\top} B \hat{\theta}_\tau  + \beta_\tau \|B x \|_{V_\tau^{-1}}   \leq C,
\end{align} and since we know that $x_\star \in \mathcal{X}$, and $\theta_\star \in \mathcal{E}_t$ with high probability, we can write  \begin{align}
 x_\star^{\top} B  \theta_\star + 2\beta_\tau \| B x_\star \|_{V_\tau^{-1}}  \geq \ x_\star^{\top} B \hat{\theta}_\tau + \beta_\tau \| B x_\star \|_{V_\tau^{-1}}  > C. \label{SCLUCB2:toboundthenumber:toshowlowerboundonlambdamin1}
\end{align} Then, using the LHS and RHS of \eqref{SCLUCB2:toboundthenumber:toshowlowerboundonlambdamin1}, we can get \begin{align}
 \frac{2 L \|B \|\beta_\tau}{\sqrt{\lambda_{\text{min}}(V_\tau)}} \geq 2 \beta_\tau \| x_\star \|_{V_\tau^{-1}} \geq \Delta,  \nonumber
\end{align} and hence the following upper bound on minimum eigenvalue of the Gram matrix: \begin{align}
    \lambda_{\text{min}}(V_\tau) < \left(\frac{2 L \|B\| \beta_\tau }{\Delta}\right)^2. \nonumber
\end{align}

Therefore, at any round $\tau$ that a conservative action is played, whether it is because $\{ F=0\}$ happens or beccause we have $\{ \lambda_{\text{min}}(V_\tau) < \left(\frac{2 L \|B\| \beta_\tau }{\Delta}\right)^2 \}$, we can always conclude that  \begin{align}
\lambda_{\text{min}}(V_\tau) < \left(\frac{2 L \|B\| \beta_\tau }{\Delta}\right)^2 \label{sclucb2:findingupperbound_on_lambdamin_forevernntc}
\end{align}

The remaining of the proof builds on two auxiliary lemmas. First, in  Lemma \ref{SCLUCB2:lemma:lowerbounding_lambdamin-vt_WITH_numberoftimesconservative}, we show that the minimum eigenvalue of the Gram matrix $V_t$ is lower bounded with the number of times SCLUCB2 plays the conservative actions.

\begin{lemma}\label{SCLUCB2:lemma:lowerbounding_lambdamin-vt_WITH_numberoftimesconservative}
On event $\{ \theta_\star \in \mathcal{E}_t \}$, it holds that \begin{align}
    \mathbb{P}(  \lambda_{\text{min}} (V_{t}) \leq t ) \leq d \exp{ \left( - \frac{(\rho^2 \sigma_{\zeta}^2 |N_t^c|  - t)^2}{32 \rho^4 |N_t^c| } \right)},
\end{align} where  $\rho = \frac{C}{\|B\| S}$.
\end{lemma}
Using \eqref{sclucb2:findingupperbound_on_lambdamin_forevernntc} and applying Lemma \ref{SCLUCB2:lemma:lowerbounding_lambdamin-vt_WITH_numberoftimesconservative}, it can checked that with probability $1-\delta$ \begin{align}
 \left(\frac{2 L \|B\| \beta_\tau }{\Delta}\right)^2   > \rho^2 \sigma_{\zeta}^2  |N_\tau^c|  - \sqrt{32 \rho^4. \nonumber |N_\tau^c|\log(\frac{d}{\delta})},
\end{align}
Then using Lemma \ref{algebra_for_upperbound}, we can conclude the following upper bound \begin{align}
    |N_\tau^c| \leq \left(\frac{2 L S \| B\|^2\beta_\tau }{C \Delta \sigma_\zeta }\right)^2 + \frac{32 \log(\frac{d}{\delta})}{\sigma_\zeta^4} + \frac{8 L S \| B\|^2 \beta_\tau \sqrt{2 \log(\frac{d}{\delta})}}{C \Delta \sigma_\zeta^3}. \nonumber
\end{align}
\end{proof}

\subsection{Proof of Lemma \ref{SCLUCB2:lemma:lowerbounding_lambdamin-vt_WITH_numberoftimesconservative}}

Our objective is to establish a lower bound on  $\lambda_{\text{min}} (V_{t})$ for all $t$. It holds that \begin{align}
    V_t &= \lambda I + \sum_{s=1}^t x_s x_s^{\top} \nonumber \\& \succeq \sum_{s \in N_t^c} \left(  \rho \zeta_s \right) \left( \rho \zeta_s \right)^{\top} \nonumber\\&
    = \sum_{s \in N_t^c} \bigg(  \rho^2 \mathbb{E} [\zeta_s \zeta_s^{\top}]    +   \rho^2 \zeta_s \zeta_s^{\top} - \rho^2 \mathbb{E} [\zeta_s \zeta_s^{\top}] \bigg) \nonumber\\&
    \succeq \rho^2 \sigma_{\zeta}^2 |N_t^c|  I + \sum_{s \in N_t^c} G_s,
\end{align} where $G_s$ is defined as \begin{align}
    G_s = \bigg(   \rho^2 \zeta_s \zeta_s^{\top} - \rho^2 \mathbb{E} [\zeta_s \zeta_s^{\top}] \bigg). \label{SCLUCB2:defnition_of:G_s}
\end{align}
Thus, using Weyl's inequality, it follows that \begin{align}
    \lambda_{\text{min}} (V_{t}) \geq \rho^2 \sigma_{\zeta}^2 |N_t^c| - \lambda_{\text{max}} (\sum_{s \in N_t^c} G_s). \nonumber
\end{align}

Next, we apply the matrix Azuma inequality (see Theorem \ref{matrix_azuma_inequality}) to find an upper bound on $\lambda_{\text{max}} (\sum_{s \in N_t^c} G_s)$. 
For this, we first need to show that the sequence of matrices $G_s$ satisfies the conditions of Theorem \ref{matrix_azuma_inequality}. By definition of $G_s$ in \eqref{SCLUCB2:defnition_of:G_s}, it follows that $\mathbb{E}[G_s | \mathcal{F}_{s-1}] = 0$, and $G_s^{\top} = G_s$. Also, we construct the sequence of deterministic matrices $A_s$ such that $G_s^2 \preceq A_s^2$ as follows. We know that for any matrix $K$,  $K^2 \leq \|K\|_2^2 I$, where $\| K\|_2$ is the maximum singular value of  $K$, i.e.,\begin{align}
    \sigma_{\text{max}}(K) = \max_{\|u\|_1 = \|v\|_2 = 1} u^{\top} K v. \nonumber
\end{align} Thus, we first show the following bound on the maximum singular value of the matrix $G_s$ defined in \eqref{SCLUCB2:defnition_of:G_s}:  \begin{align}
   \max_{\|u\|_1 = \|v\|_2 = 1} u^{\top} G_s v  & =   \rho^2 (u^\top \zeta_s) (v^\top \zeta_s)^{\top} - \rho^2 \mathbb{E} \left[(u^\top \zeta_s) (v^\top \zeta_s)^{\top}\right] \nonumber\\& \leq   \rho^2 \|\zeta_s\|_2^2 + \rho^2 \mathbb{E} \left[\|\zeta_s\|_2^2 \right] \nonumber \\& \leq  2 \rho^2, \nonumber
\end{align} where we have used Cauchy-Schwarz inequality and the last inequality comes from the fact that $\| \zeta_s\|_2 = 1$ almost surely.
From the derivations above, and choosing $A_s = 2 \rho^2 I$,  it almost surely holds that $G_s^2 \preceq \sigma_{\text{max}}(G_s)^2 I  \preceq 4 \rho^4 I = A_s^2$. Moreover, using triangular inequality, it holds that \begin{align}
    \| \sum_{s \in N_t^c} A_s^2 \| \leq \sum_{s \in N_t^c} \| A_s^2 \| \leq 4 \rho^4 |N_t^c|. \nonumber 
\end{align}

Now we can apply the matrix Azuma inequality, to conclude that for any $ c  \geq 0 $, \begin{align}
    \mathbb{P} \left(  \lambda_{\text{max}} (\sum_{s \in N_t^c} G_s)  \geq c    \right) \leq d \exp{\left( - \frac{c^2}{32 \rho^4|N_t^c|} \right)}. \nonumber
\end{align} Therefore, it holds that with probability $1-\delta$, $\lambda_{\text{max}} (\sum_{s \in N_t^c} G_s) \leq \sqrt{32\rho^4 |N_t^c| \log(\frac{d}{\delta})}$, and hence with probability $1-\delta$, \begin{align}
    \lambda_{\text{min}} (V_{t}) \geq \rho^2 \sigma_{\zeta}^2  |N_t^c|  - \sqrt{32 \rho^4 |N_t^c|\log(\frac{d}{\delta})}, \label{SCLUCB2:TRYT0lowerboundinglambdaofmin}
\end{align} or equivalently, 
\begin{align}
    \mathbb{P}(  \lambda_{\text{min}} (V_{t}) \leq t ) \leq d \exp{ \left( - \frac{(\rho^2 \sigma_{\zeta}^2 |N_t^c|  - t)^2}{32 \rho^4 |N_t^c| } \right)},
\end{align} where  $\rho = \frac{C}{\|B\| S}$. This completes the proof. 

\subsection{Simulation Results}
 
 In order to verify our results  on the regret bound of SCLUCB2, we plot the Figure \ref{fig:comparinfwithsafelucb} which plots the cumulative regret of the two algorithms averaged over $100$ realizations. Therefore, the regret of SCLUCB2 matches the proposed problem-dependent upper bound in \cite{amani2019linear}.
 
 \begin{figure}
     \centering
          \includegraphics[width=0.5\textwidth]{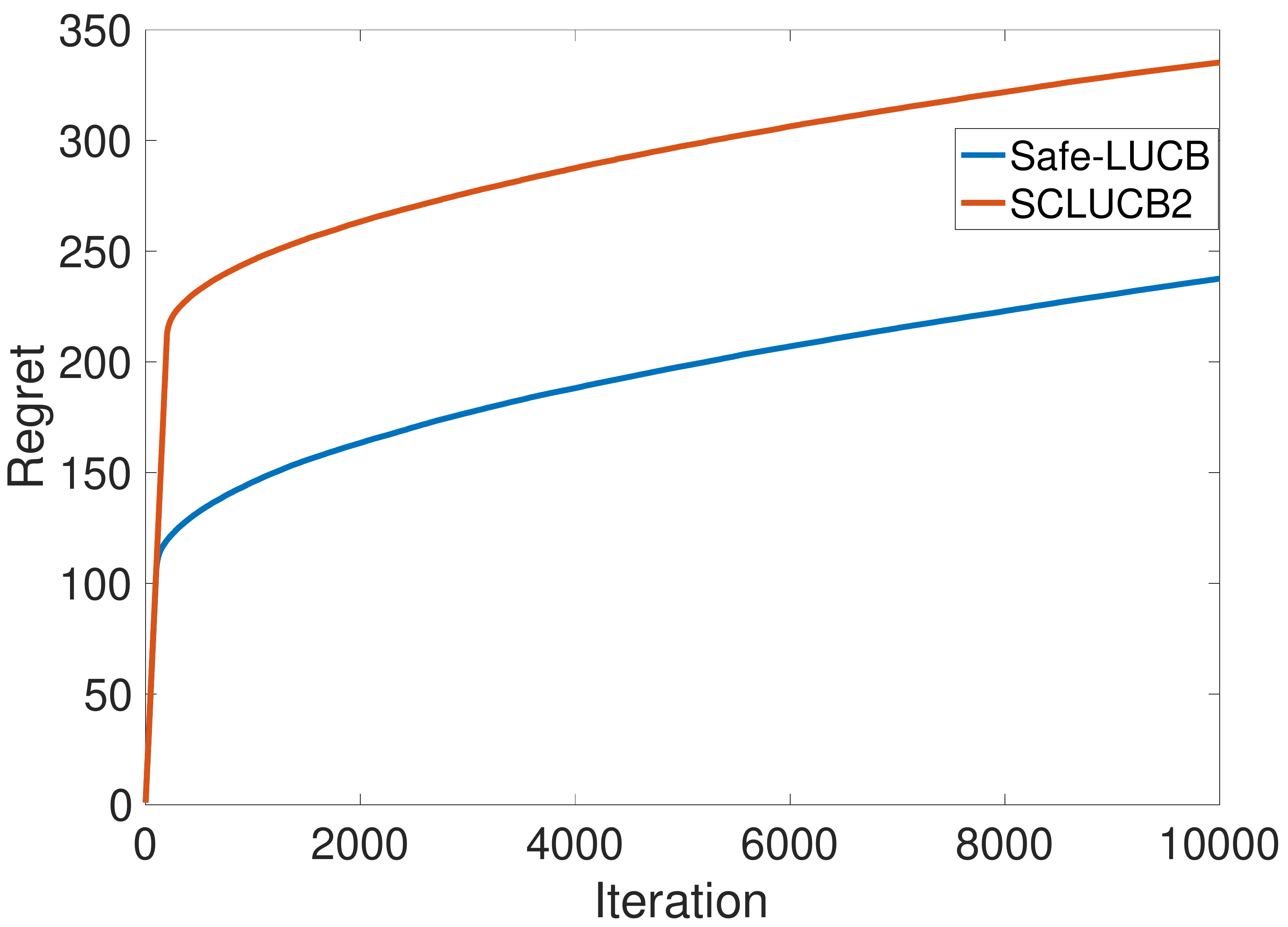}
         \caption{ Cumulative regret of SCLUCB2 versus Safe-LUCB in \cite{amani2019linear} averaged over $100$ realizations. }\label{fig:comparinfwithsafelucb}
   \end{figure}

\end{document}